%% file: main_v3.tex
\documentclass{article} 
\usepackage{iclr2027_conference,times}

\input{math_commands.tex}

\usepackage{hyperref}
\usepackage{url}
\usepackage{graphicx}
\usepackage{wrapfig}
\usepackage{booktabs}
\usepackage{amsfonts}
\usepackage{nicefrac}
\usepackage{microtype}
\usepackage{xcolor}
\usepackage{algorithm}
\usepackage{algpseudocode}
\usepackage{caption}
\usepackage{amsthm, amssymb}
\usepackage{thmtools,thm-restate}
\usepackage{colortbl}
\usepackage{makecell}
\usepackage{multirow}

\newif\ifshowrev\showrevtrue
\definecolor{revred}{RGB}{200,30,30}

\newcommand{\prev}[1]{#1} 

\definecolor{okgreen}{RGB}{21,128,61}
\definecolor{oursbg}{RGB}{235,240,250} 
\newcommand{\cmark}{{\color{okgreen}\checkmark}}
\newcommand{\xmark}{{\color{red!75!black}$\times$}}

\def\eqref#1{Eq.~(\ref{#1})}

\title{{Adjoint Guidance Flow:\\Amortized Critic Guidance for VLA Policies}}

\iclrfinalcopy

\author{%
Jeongsol Kim$^{1}$, Youngjun Jun$^{1}$, Kyumin Choi$^{2}$, Youngmin Kim$^{1}$,  Seonghyun Jin$^{1}$\\ \textbf{Sunwoo Park$^{1}$,  Jangho Park$^{1}$, Kwanyoung Kim$^{3*}$, Jong Chul Ye$^{1*}$} \\
$^{1}$KAIST \quad $^{2}$SKKU \quad $^{3}$GIST \qquad $^{*}$Co-corresponding authors
}

\newtheorem{lemma}{Lemma}
\newtheorem{remark}{Remark}

\begin{document}

\maketitle
\lhead{Preprint}

\begin{figure}[h!]
    \centering
    \vspace{-0.7cm}
    \includegraphics[width=0.9\linewidth]{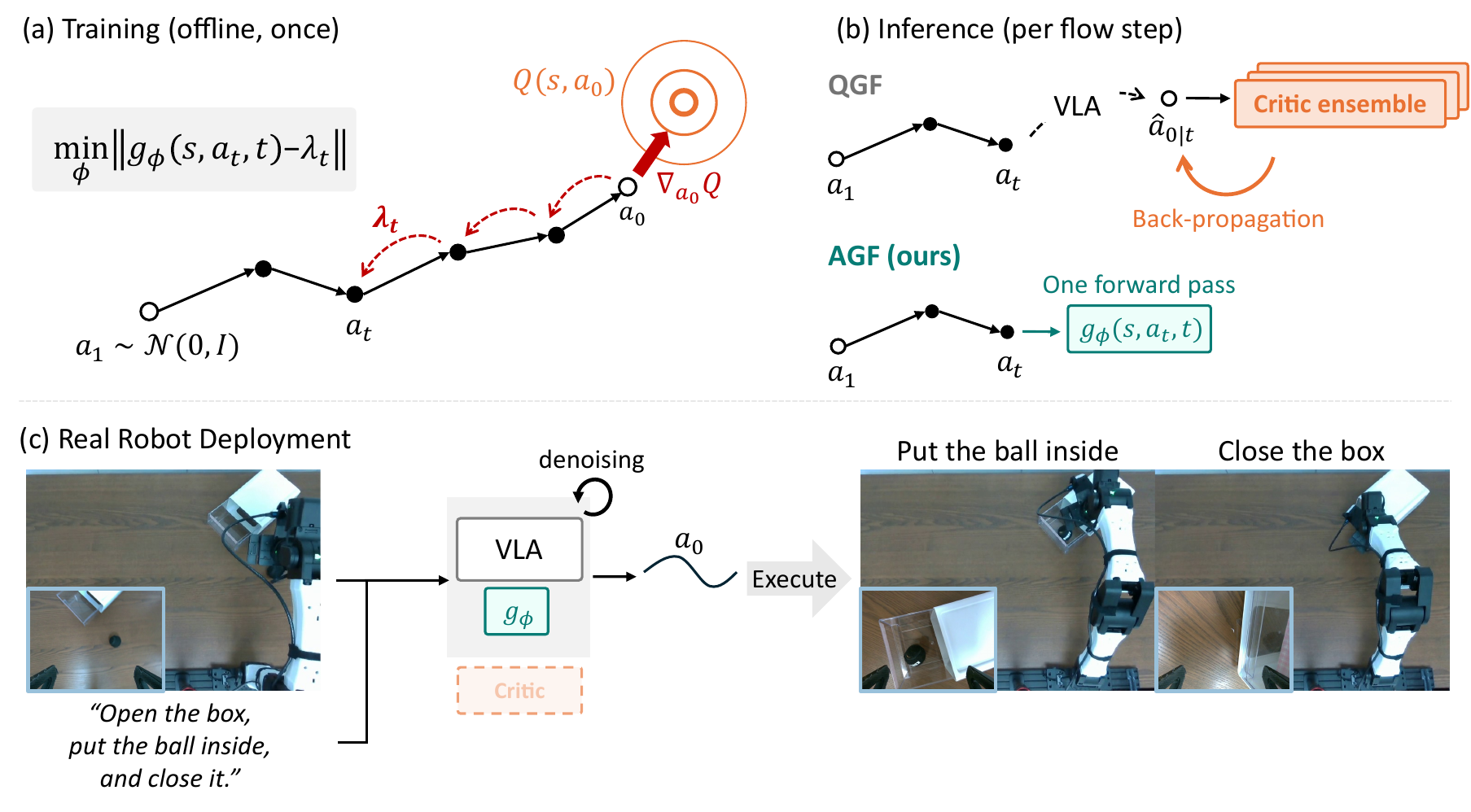}
    \vspace{-0.4cm}
    \caption{
    \textbf{Overview of AGF.} (a) Training: {the terminal critic gradient is carried back through the frozen flow and regressed into $g_\phi$.} (b) Inference: QGF back-propagates a critic ensemble {per} step; AGF runs {one} forward pass of $g_\phi$. (c) Deployment {on a real robot, with no critic ensemble on board.}
}
    \label{fig:method}
\end{figure}

\begin{abstract}
Flow-based Vision-Language-Action (VLA) policies are typically trained by behavior cloning and thus do not explicitly optimize long-term task return. Critic guidance steers generation toward higher-value actions, but existing methods differentiate the critic through a one-step surrogate of the sampler and back-propagate a critic ensemble at every flow step.
In contrast, here we propose Adjoint Guidance Flow (AGF), which amortizes trajectory-aware critic guidance into a lightweight guidance network while preserving the pretrained VLA policy.
Specifically, we formulate critic-guided flow generation as a deterministic optimal control problem, whose optimal guidance is a costate that carries the terminal critic gradient back through the remaining flow, and regress the guidance network onto this costate while keeping both the VLA and critic frozen.
This design provides favorable memory and throughput scaling during training, and inference needs one guidance-network forward pass per step, without the critic ensemble, back-propagation, or adjoint computation.
Across LIBERO, RoboCasa, and LIBERO-Pro, AGF consistently improves pretrained VLAs, remains competitive with critic-guidance and policy-fine-tuning baselines, and is the most robust method when a single guidance strength is deployed across tasks.
Compared with QGF, AGF runs $3.6\times$ faster per guidance step with $7.0\times$ fewer parameters, with comparable and even better performance, showing that critic guidance can be trajectory-aware and lightweight.

Project page: \url{https://jeongsol-kim.github.io/agf/project_page/}
\end{abstract}

\section{Introduction}

\begin{table}[H]
\centering
\vspace{-0.3cm}
\caption{
\textbf{Comparison of critic-guidance paradigms.}
{Pointwise guidance: DPS~\citep{chung2023diffusion}, QGF~\citep{zhou2026test}, QPILOTS~\citep{ruan2026qpilots}, GAF~\citep{yang2026guided}; adjoint matching: AM~\citep{domingo-enrich2025adjoint}, QAM~\citep{li2026qlearning}; $\hat{\va}_{0|t}$ is the Tweedie estimate.}
}
\label{tab:comparison}
\vspace{-0.5em}
\footnotesize
\setlength{\tabcolsep}{5pt}
\renewcommand{\arraystretch}{1.15}
\resizebox{0.95\linewidth}{!}{%
\begin{tabular}{@{}lcccc@{}}
\toprule
Paradigm
& \makecell{Inference-time\\control}
& \makecell{Guidance\\signal}
& \makecell{Guidance computation\\at inference}
& \makecell{Trained\\component} \\
\midrule
Pointwise guidance
& \cmark
& $\nabla_{\va_t}Q^\pi(\vs, \hat{\va}_{0|t})$
& critic forward + backward
& -- \\

Adjoint matching
& \xmark
& $\nabla_{\va_t} Q^\pi(\vs, \va_0)$
& none
& policy \\
\midrule

\textbf{AGF (Ours)}
& \cmark
& $\nabla_{\va_t} Q^\pi(\vs, \va_0)$
& guidance forward
& guidance net \\
\bottomrule
\end{tabular}%
}
\end{table}

Vision-Language-Action (VLA) models have recently emerged as a promising paradigm for learning general-purpose robot policies by leveraging the rich visual and semantic representations of vision-language models (VLMs)\prev{~\citep{pmlr-v229-zitkovich23a, kim2024openvla, intelligence2025pi_}}.
Given visual observations and language instructions, VLAs aim to generate temporally coherent action for accomplishing diverse manipulation tasks, often predicted in chunks ~\citep{Zhao-RSS-23, intelligence2025pi_}.
A central challenge in robot policy learning is the inherently multimodal distribution of feasible actions~\citep{pmlr-v164-mandlekar22a}: regression-based behavior cloning can average across modes, producing actions that match no demonstrated behavior~\citep{zhang2018deep, florence2022implicit}. Diffusion and flow policies address this by modeling the conditional action distribution through iterative generation~\citep{chi2025diffusion}, but typically rely on task-specific visual representations without the broad semantic understanding of pretrained VLMs.

Recent VLA models combine the two: {rather than discretizing actions into tokens as in OpenVLA~\citep{kim2024openvla}, which can interfere with the VLM's pretrained representations, they pair} a pretrained VLM with a diffusion- or flow-based action expert that models the continuous action distribution~\citep{intelligence2025pi_, shukor1844smolvla, bjorck2025gr00t}.
{Despite their expressive action distributions, these policies are commonly trained with behavior cloning~\citep{chi2025diffusion}, which neither optimizes task return nor distinguishes high-value actions from actions that are merely likely under the demonstrations~\citep{wang2023diffusion,psenka2024learning,wang2026q}.}
{Recent works therefore incorporate a learned critic of long-term utility, either adapting the policy toward high-value actions through additional training~\citep{doo2026qflow,shi2026flowdpg,wang2026q} or keeping it fixed and applying critic guidance during generation}\prev{, such as QGF~\citep{zhou2026test}, which steers each generation step with critic gradients~\citep{ruan2026qpilots}}.
The former does not require the critic during inference, but requires additional policy optimization for the downstream task{, whereas} inference-time guidance can incorporate task-specific value information while preserving the pretrained generative policy.

However, existing inference-time critic-guidance methods~\citep{zhou2026test, ruan2026qpilots, yang2026guided} suffer from two main limitations. First, they {differentiate the critic only at the Tweedie estimate, a one-step surrogate of the sampler, and thus ignore how the remaining flow shapes the executed action.}
Second, they require repeated critic back-propagation at every generation step, whose cost is further amplified when a critic ensemble is used for conservative value estimation.
Fine-tuning approaches such as adjoint matching~\citep{domingo-enrich2025adjoint} and QAM~\citep{li2026qlearning} avoid the critic at inference by amortizing adjoint supervision into the policy itself, but require policy back-propagation and modify its weights for each downstream task.

To bridge these alternatives, we propose Adjoint Guidance Flow (AGF), which amortizes trajectory-wise critic adjoints into a separate, lightweight guidance network.
We formulate critic-guided flow generation as a deterministic optimal control problem, where the optimal guidance {at each flow step is a costate: the sensitivity of the final critic value to the current intermediate action through all remaining flow steps.}
We then train the guidance network to directly predict this costate from intermediate flow states.
Once trained, AGF {guides the frozen policy} with a single guidance-network forward pass, without {the critic ensemble} or modifying the pretrained policy (Table~\ref{tab:comparison}){; compared with QGF, this is $3.6\times$ faster per guidance step with $7.0\times$ fewer parameters.} {Our contributions are summarized as follows.}
\begin{list}{\textbullet}{\setlength{\topsep}{2pt}\setlength{\partopsep}{0pt}\setlength{\itemsep}{0pt}\setlength{\parsep}{0pt}\setlength{\leftmargin}{2.5em}\setlength{\labelwidth}{1em}}
\item {To guide a frozen VLA without evaluating or back-propagating the critic at inference, we propose AGF, which amortizes trajectory-aware critic guidance into a lightweight network.}
\item {We characterize the optimal guidance as the costate of a deterministic optimal control problem and regress the guidance network onto costates computed along its own trajectories, smoothed by particle-averaged Jacobians whose noise adds no bias for any number of particles.}
\item {Across LIBERO, RoboCasa, and LIBERO-Pro with three flow-based VLAs and on a real robot, AGF improves every pretrained policy, remains competitive with critic guidance and policy fine-tuning, and keeps its gains under one shared strength.}
\end{list}

\section{Background}

\subsection{Flow-based Model}
Flow-based generative models transport a noise sample ${\va_1} \sim p_1$ to a data sample ${\va_0} \sim p_0$ along the ODE $d{\va_t} = \vv_t({\va_t}) dt$, where the velocity field is parameterized by a neural network ${\vv_\theta}$ trained with the conditional flow matching objective~\citep{lipman2023flow},
\begin{equation}
    \min_\theta  \mathbb{E}_{t,{\va_0} \sim p_0,{\va_1} \sim p_1} \left\| {\vv_t(\va_t | \va_0) - \vv_\theta(\va_t)} \right\|^2,
\end{equation}
where ${\va_t = (1-t)\va_0 + t\va_1}$ is the linear interpolant and ${\vv_t(\va_t|\va_0) = \va_1 - \va_0}$ denotes the conditional target. Samples are generated by integrating the ODE backward from $t=1$ to $t=0$. The posterior mean at time $t$, known as Tweedie's estimate~\citep{efron2011tweedie, kim2021noise2score}, is ${\mathbb{E}[\va_0 | \va_t] = \va_t - t \vv_\theta(\va_t)}$. {Hereafter, $\va$ denotes an action chunk generated conditioned on the state $\vs$, with velocity $\vv_\theta(\va_t,\vs,t)$.}

\subsection{Action-Value Function}
Consider a trajectory $\tau = \{(\vs^{(k)}, \va^{(k)})\}_{k=1}^{{H}}$, where $\vs^{(k)} \in \mathcal{S}$ and $\va^{(k)} \in \mathcal{A}$ denote the state and action at time step $k$, respectively, and ${H}$ denotes the finite horizon. 
A policy $\pi(\va^{(k)}|\vs^{(k)})$ specifies the distribution of actions conditioned on the current state, and a reward function $r_k = r(\vs^{(k)}, \va^{(k)})$ assigns a scalar reward at each time step.
We define the discounted return from time step $k$ as $G_k = \sum_{i=k}^{{H}} \gamma^{i-k} r_i$ where $\gamma\in[0,1]$ denotes the discount factor.
The action-value function under policy $\pi$ is then defined as 
\begin{equation}
    Q^\pi(\vs^{(k)}, \va^{(k)}) = \mathbb{E}_\pi [ G_k | \vs^{(k)}, \va^{(k)}],
\label{eqn:critic}
\end{equation}
which represents the expected return obtained by taking action $\va^{(k)}$ at state $\vs^{(k)}$ and following the policy $\pi$ thereafter~\citep{sutton1998reinforcement, murphy2024reinforcement}. In our setting, the reward is sparse and assigned only at the terminal state. In other words, $r_k=0$ for $k<{H}$ and $r_{{H}}=r(\vs^{({H})})$.
For a fixed policy, the critic is trained with a SARSA~\citep{rummery1994line} style temporal-difference objective~\citep{sutton1988learning},
\begin{equation}
    \min_\omega \mathbb{E}\left [ (Q_\omega(\vs^{(k)}, \va^{(k)}) - y_k)^2\right], \qquad y_k = r_k + \gamma Q_{\bar\omega} (\vs^{(k+1)}, \va^{(k+1)}),
\label{eqn:sarsa_td}
\end{equation}
where $Q_{\bar\omega}$ denotes an EMA-updated target critic and the next action $\va^{(k+1)} \sim \pi(\cdot|\vs^{(k+1)})$ is taken from the same policy rollout. {This critic is frozen after training and is the only value signal in the rest of the paper; every guidance method we compare, AGF included, uses its action gradient.}

\subsection{Critic Guidance}
Recent flow-based VLAs model action generation as a conditional flow process. We use $k$ for the environment step and $t\in[0,1]$ for the continuous flow time. At step $k$, the policy transforms a noisy action chunk $\va_1^{(k)}$ into a clean action $\va_0^{(k)}$ through the learned velocity field.
Since the critic $Q^\pi(\vs,\va)$ is defined on the clean action at $t=0$, critic guidance at $\va_t$ ideally follows
\begin{equation}
    \vg_t = \nabla_{\va_t} \mathbb{E}_{\va_0\mid\va_t} \left[ Q^\pi(\vs,\va_0) \right],
\end{equation}
where we omit the environment-step superscript for simplicity. Monte-Carlo estimation requires sampling actions from $p(\va_0\mid\va_t)$ and evaluating the critic for each sample, which is often computationally expensive.
DPS-type approximation \citep{chung2023diffusion} replaces the expectation of the critic with the critic evaluated at the posterior mean,
\begin{equation}
    \vg_t \approx \nabla_{\va_t} Q^\pi(\vs,\hat\va_{0|t}), \qquad \hat\va_{0|t} = \va_t-t\vv_\theta(\va_t,t),
\end{equation}
where $\hat\va_{0|t}$ is the Tweedie estimate of the clean action.

Recent methods including QGF~\citep{zhou2026test},
QPILOTS~\citep{ruan2026qpilots}, and GAF~\citep{yang2026guided} further
{replace the Jacobian of the Tweedie map, $\partial\hat\va_{0|t}/\partial\va_t=\rmI-t\partial\vv_\theta(\va_t,t)/\partial\va_t$, by the identity, keeping only the critic gradient at $\hat\va_{0|t}$,}
\begin{equation}
    \nabla_{\va_t}Q^\pi(\vs,\hat\va_{0|t})
    =
    \prev{\left(\frac{\partial\hat\va_{0|t}}{\partial\va_t}\right)^{\top}}
    \frac{\partial Q^\pi(\vs,\hat\va_{0|t})}
         {\partial\hat\va_{0|t}}
    \approx
    \frac{\partial Q^\pi(\vs,\hat\va_{0|t})}
         {\partial\hat\va_{0|t}}.
\label{eqn:qgf}
\end{equation}
This avoids back-propagation through the flow model, but ignores how
perturbations at intermediate flow states propagate through the remaining
generation trajectory.

\section{Adjoint Guidance Flow}
\label{sec}

\subsection{Optimal Control Problem}

Consider a pretrained flow policy $\vv_\theta$ that generates a clean action chunk $\va_0 \sim \pi_\theta(\va_0 \mid \vs)$ by integrating the corresponding flow ODE, and an action-value function $Q^\pi(\vs,\va_0)$ trained as in \eqref{eqn:sarsa_td}.
Critic guidance steers the flow trajectory toward higher-value actions by augmenting the pretrained dynamics with an additive control,
\begin{equation}
    d\va_t=\left[\vv_\theta(\va_t,\vs,t)+\vu_t\right]dt,
    \label{eqn:controlled_ode}
\end{equation}
where $\vu_t$ denotes the guidance at flow time $t$.
Rather than constructing $\vu_t$ from pointwise critic gradients at the Tweedie estimate as in \eqref{eqn:qgf}, we formulate it as a deterministic optimal control problem over the generation trajectory, targeting the value of the action it will actually produce.


For a given state $\vs$ and initial noise sample, the time-dependent control $\vu_t$ is optimized by
\begin{equation}
    \vu^\star = \argmax_\vu Q^\pi (\vs, \va_0) - \int_0^1 \frac{1}{2\beta_t} \| \vu_t \|^2 dt,
\label{eqn:oc}
\end{equation}
which is subject to the controlled flow dynamics~\eqref{eqn:controlled_ode}, whose integration from the initial noise determines the terminal action $\va_0$.
The two terms represent complementary objectives. The terminal value $Q^\pi(\vs, \va_0)$ drives the generated action toward high expected return, while the quadratic control cost penalizes the energy of the deviation from the pretrained flow, so that as $\beta_t \to 0$ the controlled dynamics reduce to the pretrained policy. The coefficient $\beta_t$ thus interpolates between pure imitation and pure value seeking.
The following result characterizes the optimal guidance under our generative-time convention, where the flow evolves from $t=1$ (noise) to $t=0$ (clean action).
\begin{restatable}[]{prop}{optimalcontrol}
\label{prop:optimal}
Under the controlled dynamics $d\va_t = [\vv_\theta(\va_t, \vs, t) + \vu_t]dt$, the optimal control of \eqref{eqn:oc} satisfies
\begin{equation}
    \vu_t^\star = -\beta_t \vlamb_t,
\label{eqn:optimal_control}
\end{equation}
where the costate $\vlamb_t$ is initialized by $\vlamb_0 = \nabla_{\va_0}Q^\pi(\vs, \va_0)$ and evolves according to
\begin{equation}
\frac{d\vlamb_t}{dt} = -\left( \frac{\partial \vv_\theta(\va_t,\vs, t)} {\partial \va_t} \right)^\top \vlamb_t.
\label{eqn:costate}
\end{equation}
\end{restatable}
The result follows from the standard forward-time optimal-control formulation under the reparameterization $\tau=1-t$; we provide the derivation in Appendix~\ref{app:prop1}.
Proposition~\ref{prop:optimal} has a simple interpretation. Integrating the costate dynamics in \eqref{eqn:costate} from the terminal condition gives the closed form
\begin{equation}
    \vlamb_t = \left(\frac{\partial \va_0}{\partial \va_t}\right)^\top \nabla_{\va_0} Q^\pi(\vs, \va_0),
    \label{eqn:costate_closed}
\end{equation}
where $\partial \va_0 / \partial \va_t$ is the Jacobian of the map that carries the intermediate state $\va_t$ to the terminal action $\va_0$ under the controlled dynamics. The costate therefore measures how the critic value of the final action responds to a perturbation of the current state.
The optimal control $\vu^\star_t=-\beta_t\vlamb_t$ steers each intermediate state in the direction that most increases the critic value of the action it will produce, rather than the value of the Tweedie estimate.

A direct approach is to amortize the adjoint into the flow policy itself, as in adjoint matching~\citep{domingo-enrich2025adjoint}, which fine-tunes the policy under a memoryless stochastic sampler.
For large pretrained VLAs, however, this updates the policy parameters at substantial memory and compute cost, may alter the pretrained behavior, and fixes the guidance strength into the fine-tuned weights.
Instead, we train a lightweight guidance network $g_\phi$ to approximate $\vlamb_t$ while keeping the VLA frozen, so that no VLA parameter gradients are computed and the guidance can be attached or removed without modifying the policy.

\begin{figure}
    \centering
    \includegraphics[width=0.9\linewidth]{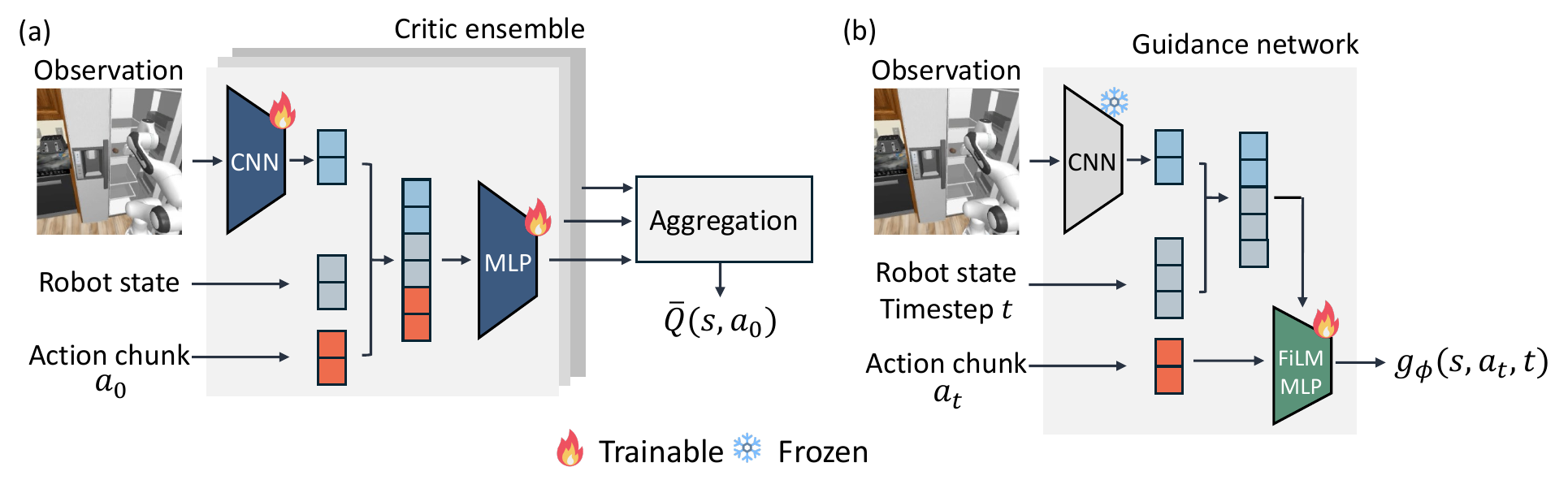}
    \caption{{Guidance-network architectures.} (a) The critic ensemble predicts $Q^\pi(\vs,\va_0)$ from the observation, robot state, and clean action chunk. (b) {Ours: $g_\phi$} reuses {one} frozen critic encoder and {predicts} the trajectory-wise adjoint $g_\phi(\vs,\va_t,t)$ {with a FiLM-conditioned MLP}.}
    \label{fig:arc}
    \vspace{-0.6cm}
\end{figure}

\subsection{Guidance Network Training}
\label{subsec:train}
Concretely, our guidance network $g_\phi$ reuses a frozen visual encoder of one critic ensemble member and predicts the costate with a lightweight FiLM-modulated MLP~\citep{perez2018film} conditioned on the visual feature, flow timestep, and robot state where available (Figure~\ref{fig:arc}b).

Training $g_\phi$ requires a costate target at each intermediate action state, built in two steps: generate a controlled flow trajectory with the current guidance network, then propagate the terminal critic gradient backward through its transitions.
The control $\vu_t = -\beta_t g_\phi(\vs, \va_t, t)$ gives the Euler transition
\begin{equation}
    \va_{t+\Delta t} = \va_t + \Delta t \left[ \vv_\theta(\va_t, \vs, t) - \beta_t g_\phi(\vs, \va_t, t) \right], \qquad \Delta t < 0.
\label{eqn:euler}
\end{equation}

After generating the trajectory from $t=1$ to $t=0$, we initialize the terminal costate as $\vlamb_0 = \nabla_{\va_0} Q^\pi(\vs, \va_0)$ and propagate it backward through the discretized transitions. Consistent with the optimal-control formulation, the realized control $\vu_t$ is held fixed during this propagation, so each transition contributes the Jacobian $\partial \va_{t+\Delta t}/\partial \va_t |_{\vu_t} = \rmI + \Delta t (\partial\vv_\theta/\partial\va_t)$, giving the recursive update
\begin{equation}
    \vlamb_t = \left( \rmI + \Delta t {J_t(\va_t)} \right)^{\top} \vlamb_{t+\Delta t}, {\qquad \text{where} \quad J_t(\va_t) := \frac{\partial \vv_\theta(\va_t, \vs, t)}{\partial \va_t}}.
\label{eqn:adjoint}
\end{equation}
This recursion is the discrete counterpart of the continuous costate dynamics in \eqref{eqn:costate}{, and $-\beta_t\vlamb_t$ is the optimal control~\eqref{eqn:optimal_control} of the discretized dynamics}. {Therefore,} we train the guidance network by matching these trajectory-wise adjoint targets,
\begin{equation}
    \mathcal{L}_{\text{AGF}}(\phi) = \mathbb{E}_{\vs, \va_1 \sim \mathcal{N}(0,\rmI), t} \left[ \left\| g_\phi(\vs, \va_t, t) - \text{sg}[\vlamb_t] \right\|_2^2 \right],
\label{eqn:agf_loss}
\end{equation}
where $\text{sg}[\cdot]$ denotes stop-gradient.
{Since the optimal control is $-\beta_t\vlamb_t$, we regress $g_\phi$ onto $\vlamb_t$ itself and set $\beta_t=1$ during training; at inference the guidance scale is exposed as a weight $w$, deploying $\vu_t=-wg_\phi(\vs,\va_t,t)$, so the strength can be adjusted without retraining.}

Importantly, the guidance network affects the adjoint targets through the trajectory it induces, but its Jacobian is not included in the adjoint recursion. This is not an approximation but a consequence of the optimal-control formulation: in Pontryagin's principle{~\citep{pontryagin1987mathematical}}, the adjoint is defined along the dynamics with the realized control held fixed, whereas including the Jacobian of $g_\phi$ would compute the sensitivity of a different, closed-loop system. The stop-gradient in \eqref{eqn:agf_loss} enforces the same separation at the loss level. 
As $g_\phi$ is updated, we regenerate the controlled trajectories and recompute their adjoint targets, yielding iterative on-policy refinement.

\subsection{Locally Regularized Adjoint Estimation}
{The adjoint target in~\eqref{eqn:adjoint} depends on the flow Jacobian $J_t(\va_t)$, which provides the exact local sensitivity along a given trajectory but can vary substantially across nearby intermediate actions in neural flow models.}
{Small changes in the guided trajectory can therefore perturb the adjoint targets, which is particularly relevant under on-policy refinement, where successive updates of the guidance network continuously shift the intermediate action trajectory.}
To obtain a more locally regularized sensitivity estimate, we perturb each intermediate action with Gaussian particles, $\va_{t, m} = \va_t+\sigma \veps_m$ where $\veps_m \sim\mathcal{N}(0,\rmI)$,
and average the corresponding flow Jacobians,
\begin{equation}
    \widehat{J}_t(\va_t) = \frac{1}{M} \sum_{m=1}^{M} J_t \left(\va_t+\sigma\veps_m\right),
\label{eqn:particle_jacobian}
\end{equation}
where $M$ denotes the number of particles and $\sigma$ controls the local perturbation scale.
Under standard regularity conditions that permit exchanging differentiation and expectation, we have $\mathbb{E}_{\veps} \left[ J_t(\va_t+\sigma\veps) \right] = \nabla_{\va_t} \mathbb{E}_{\veps} \left[ \vv_\theta(\va_t+\sigma\veps,\vs,t)\right]$.
Thus, $\widehat{J}_t(\va_t)$ is a Monte Carlo estimate of the Jacobian of a locally Gaussian-smoothed flow field, rather than an ad hoc average of pointwise Jacobians. {Because the particles are redrawn independently at each flow step, the resulting target is an unbiased estimate of the costate propagated with the smoothed Jacobian for any $M$, and $M$ only sets its variance (Lemma~\ref{prop:particle}, Appendix~\ref{app:prop4}).}
The adjoint is then propagated as
\begin{equation}
    \widehat{\vlamb}_t=\left( \prev{\rmI + \Delta t}\widehat{J}_t(\va_t)\right)^\top\widehat{\vlamb}_{t+\Delta t},
\label{eqn:particle_adjoint}
\end{equation}
with the same terminal condition $\widehat{\vlamb}_0=\nabla_{\va_0}Q^\pi(\vs,\va_0)$.
This smoothing suppresses highly localized Jacobian variations while retaining sensitivity patterns that persist in a neighborhood of the  current trajectory.
Accordingly, $\sigma$ controls a trade-off between pointwise fidelity and local regularity: as $\sigma\rightarrow0$, the estimator approaches the original pointwise adjoint, while larger $\sigma$ provides stronger smoothing of the supervision signal.
Empirically, combining Gaussian smoothing with particle averaging broadens the effective range of guidance strengths, yielding positive suite-averaged gains across all tested scales (Section~\ref{subsec:ablation}).
We use the resulting particle-averaged adjoint $\widehat{\vlamb}_t$ as the training target for the guidance network.
\begin{wrapfigure}{r}{0.52\columnwidth}
    \centering
    \vspace{-1.0cm}
    \includegraphics[width=\linewidth]{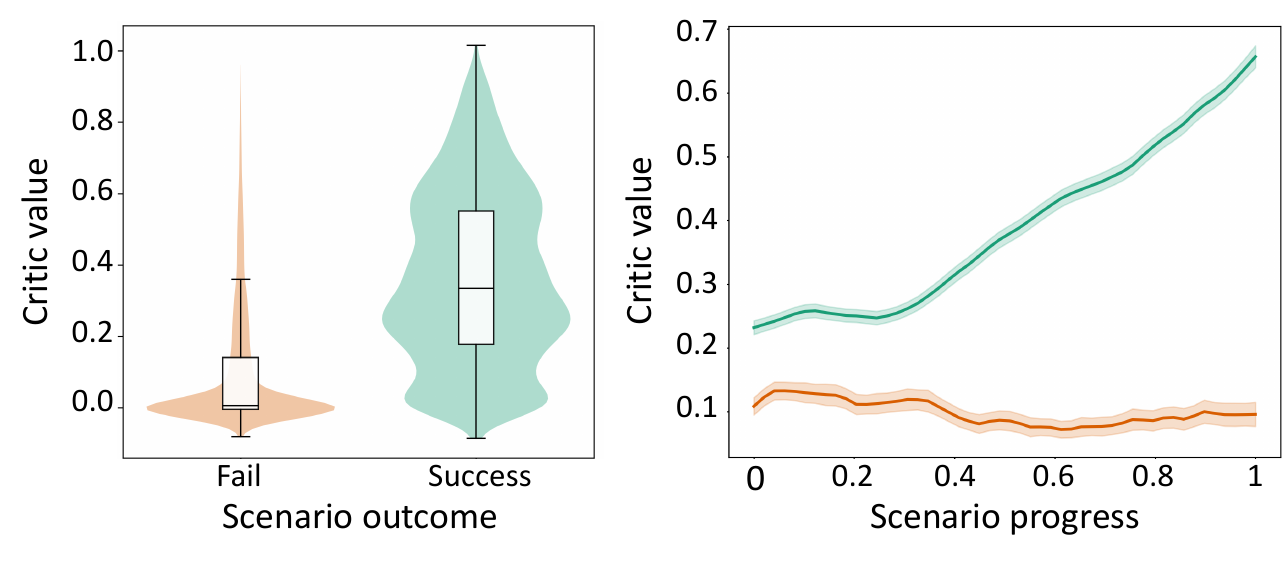}
    \vspace{-0.7cm}
    \caption{
    \textbf{Critic analysis.}
    (Left) Distribution of critic values by scenario outcome.
    (Right) Critic values over normalized scenario progress.
    Pooled across all 4 LIBERO suites (40 tasks, 1,200 episodes).
    }
    \label{fig:critic}
    \vspace{-0.5cm}
\end{wrapfigure}
\section{Experiments}
\label{sec:experiments}
We evaluate AGF on three flow-based VLAs: SmolVLA, $\pi_{0.5}$, and MolmoAct2, across LIBERO, RoboCasa, and LIBERO-Pro.
On the standard benchmarks, we evaluate SmolVLA~\citep{shukor1844smolvla} on all four LIBERO suites and $\pi_{0.5}$~\citep{intelligence2025pi_} on the atomic suite in RoboCasa.
We further evaluate MolmoAct2~\citep{fang2026molmoact2} on LIBERO-Pro, which covers the four LIBERO suites under five visual conditions, and in real-robot manipulation experiments.

\textbf{Baselines.}
We compare AGF with two inference-time critic-guidance methods QDPS and QGF\footnote{We consider QPILOTS and GAF as the same variant with QGF since they share the core update rule.}, critic-based action selection method called Q-BoN (Best-of-$N$), and policy fine-tuning method QAM. For the details on each baseline methods, please refer to the Appendix~\ref{app:baselines}.

\textbf{Implementation and evaluation.}
We evaluate SmolVLA~\citep{shukor1844smolvla} on all 10 tasks of each LIBERO suite~\citep{liu2023libero} and $\pi_{0.5}$~\citep{intelligence2025pi_} on 18 atomic RoboCasa tasks~\citep{nasiriany2024robocasa}, with 50 episodes per task and identical episode seeds for paired comparisons.
All inference-time guidance methods freeze the pretrained VLA and share the same task-specific critic, number of action-generation steps, and guidance-weight sweep range, with one weight selected per suite (Table~\ref{tab:guidance_weights}).
Following prior single-task flow-RL settings~\citep{doo2026qflow, zhou2026test}, we train a separate critic ensemble per task, and likewise a separate guidance network per task; the guidance network is trained on the same rollout data used for critic training, requiring no additional data collection. AGF denotes the particle-averaged variant ($M=4$, $\sigma=0.02$) and $M=1$ the pointwise-adjoint variant{.}
Training uses 5k rollout updates for SmolVLA and $\pi_{0.5}$ and 1k for MolmoAct2, at roughly 40 minutes per 1k updates on a single RTX 4090. QAM is trained under the same wall-clock budget. Remaining details are in Appendix~\ref{app:impl}, which also compares the QGF and AGF inference procedures (Algorithms~\ref{alg:qgf_inf}, \ref{alg:agf_inf}).

\subsection{Main Results}
\label{sec:main_results}

\textbf{Critic quality.}
We first validate whether the trained action-value function assigns meaningful values to each task. We sample actions from the policies used to train the critic and evaluate the rollout. At the end of each scenario, the rollout is labeled as either successful or failed.
Figure~\ref{fig:critic} shows that the trained critic assigns higher values to successful rollouts than to failed ones. 
For successful rollouts, the predicted value increases as the task progresses, while it remains low for failed cases. This outcome-dependent temporal behavior supports the critic as a guidance signal.

\begin{table*}[t]
\centering
\vspace{-0.4cm}
\caption{
\textbf{{Success rate (\%) on LIBERO (SmolVLA), RoboCasa ($\pi_{0.5}$), and LIBERO-Pro (MolmoAct2).}}
{Suite means over tasks, 50 episodes each; LIBERO-Pro is averaged over five visual conditions (Appendix~\ref{app:libero_pro}).}
Bold and underline mark the best and second-best inference-time method per row{; QAM fine-tunes the policy and is excluded from the ranking.}
}
\label{tab:eval}
\resizebox{0.9\textwidth}{!}{%
\begin{tabular}{cllcc|ccc>{\columncolor{oursbg}}c>{\columncolor{oursbg}}c}
\toprule
Benchmark / policy
& Suite
& Calibration
& Base
& QAM
& Q-BoN
& QDPS
& QGF
& \makecell{AGF \\(M=1)}
& \makecell{AGF \\(M=4)}\\
\midrule

\multirow{8}{*}{\makecell[c]{LIBERO / \\ SmolVLA}}
& \multirow{2}{*}{Goal}
& Suite-level
& \multirow{2}{*}{76.6}
& \multirow{2}{*}{83.6}
& 78.0 & 79.8 & \underline{80.6}
& 80.2 & \textbf{81.0} \\
&
& Task-level
&
&
& \underline{82.4} & 82.2 & 81.4
& \textbf{83.2} & \textbf{83.2} \\
\cmidrule(lr){2-10}

& \multirow{2}{*}{Object}
& Suite-level
& \multirow{2}{*}{89.6}
& \multirow{2}{*}{94.6}
& 90.6 & 89.4 & \textbf{93.6}
& 92.4 & \underline{93.4} \\
&
& Task-level
&
&
& 92.0 & 91.0 & \textbf{95.6}
& 95.0 & \underline{95.4} \\
\cmidrule(lr){2-10}

& \multirow{2}{*}{Spatial}
& Suite-level
& \multirow{2}{*}{71.6}
& \multirow{2}{*}{71.4}
& 71.2 & 68.0 & \underline{73.0}
& 72.2 & \textbf{74.6} \\
&
& Task-level
&
&
& 75.4 & 73.2 & \textbf{77.6}
& 74.8 & \underline{75.6} \\
\cmidrule(lr){2-10}

& \multirow{2}{*}{LIBERO-10}
& Suite-level
& \multirow{2}{*}{33.6}
& \multirow{2}{*}{41.0}
& 33.6 & \textbf{36.8} & 36.0
& \underline{36.2} & \textbf{36.8} \\
&
& Task-level
&
&
& 35.4 & 38.4 & 39.2
& \underline{39.6} & \textbf{42.0} \\
\midrule

\multirow{2}{*}{\makecell[c]{RoboCasa / $\pi_{0.5}$}}
& \multirow{2}{*}{Atomic}
& Suite-level
& \multirow{2}{*}{44.0}
& \multirow{2}{*}{45.8}
& 45.6 & 45.8 & \underline{46.3}
& 45.9 & \textbf{46.6} \\
&
& Task-level
&
&
& \underline{49.6} & \textbf{50.0} & \underline{49.6}
& 48.9 & 49.4 \\
\midrule

\multirow{8}{*}{\makecell[c]{LIBERO-Pro /\\ MolmoAct2}}
& \multirow{2}{*}{Goal}
& Suite-level
& \multirow{2}{*}{75.4}
& \multirow{2}{*}{53.8}
& 10.4 & \underline{76.2} & 76.0
& 76.0 & \textbf{76.4} \\
&
& Task-level
&
&
& 10.4 & \underline{77.6} & 77.4
& 77.4 & \textbf{77.8} \\
\cmidrule(lr){2-10}
& \multirow{2}{*}{Object}
& Suite-level
& \multirow{2}{*}{83.0}
& \multirow{2}{*}{59.6}
& 1.0 & \underline{84.0} & \textbf{84.6}
& \textbf{84.6} & 83.8 \\
&
& Task-level
&
&
& 1.0 & \underline{85.0} & \textbf{85.8}
& 84.8 & 84.6 \\
\cmidrule(lr){2-10}
& \multirow{2}{*}{Spatial}
& Suite-level
& \multirow{2}{*}{75.8}
& \multirow{2}{*}{58.4}
& 5.8 & \underline{76.8} & 76.0
& 76.0 & \textbf{77.4} \\
&
& Task-level
&
&
& 5.8 & \underline{78.6} & 77.0
& 78.0 & \textbf{79.0} \\
\cmidrule(lr){2-10}
& \multirow{2}{*}{LIBERO-10}
& Suite-level
& \multirow{2}{*}{65.8}
& \multirow{2}{*}{53.2}
& 0.2 & 66.8 & \textbf{67.8}
& \underline{67.2} & 67.0 \\
&
& Task-level
&
&
& 0.2 & 68.4 & \textbf{69.4}
& 68.6 & \underline{68.8} \\
\bottomrule
\end{tabular}%
}
\vspace{-0.5cm}
\end{table*}

\textbf{Task performance.}
Table~\ref{tab:eval} compares AGF with inference-time critic-guidance, action-selection, and policy-fine-tuning baselines on LIBERO, RoboCasa, and LIBERO-Pro.
For inference-time methods, we report suite-level calibration, which shares one guidance strength per suite, and task-level calibration, which selects the best strength per task from the sweep and thus serves as a per-task upper bound; strength transfer to genuinely held-out tasks is evaluated by the LOTO protocol below.
Under suite-level calibration, AGF improves the pretrained VLA by $2.6$--$4.4$ percentage points across the five suites, outperforming Q-BoN on all five and QGF on four while matching or exceeding QDPS throughout.
With task-level calibration, its gains further increase to $4.0$--$8.4$ points. Compared with QAM, AGF achieves competitive performance without modifying the pretrained policy and retains an adjustable guidance strength at inference.
The same trend holds on LIBERO-Pro with the larger MolmoAct2 backbone. Per-condition results are provided in Appendix~\ref{app:libero_pro}.
QAM, fine-tuned under the same wall-clock budget, falls substantially below the pretrained policy on this 5.5B backbone, which is consistent with its unfavorable training scaling (Figure~\ref{fig:efficiency}a), indicating that fine-tuning at this scale requires stabilization beyond matched compute.
Q-BoN is a notable exception, collapsing on MolmoAct2. We attribute this to best-of-$N$ over-optimization~\citep{gao2023scaling}: although all actions are sampled from the base policy, the critic is reliable only in regions sufficiently covered by its finite rollout data (Appendix~\ref{app:qbon}).

\textbf{Statistical significance.}
Paired tests over identical episode seeds support these comparisons: AGF significantly improves the pretrained policy ($+3.6$pp pooled over 40 LIBERO tasks, McNemar $p<10^{-3}$) and matches QGF ($+0.7$pp in AGF's favor, $p=0.54$) at a fraction of its inference cost, with the same pattern on LIBERO-Pro, and directionally, on RoboCasa. Full tests are in Appendix~\ref{app:stats}.

\textbf{Calibration transfer.}
QGF attains the strongest task-level results on the Object and Spatial suites with SmolVLA, but this ordering reverses under suite-level calibration, reflecting AGF's broad effective scale range (Figure~\ref{fig:weight}).
Leave-one-task-out (LOTO) calibration tests this directly: the strength is selected on all but one task and evaluated on the held-out task. AGF outperforms QGF under LOTO on all four suites of LIBERO and atomic tasks of RoboCasa; on four of them it selects the same strength regardless of the held-out task, so its LOTO performance coincides with suite-level calibration. Details are in Appendix~\ref{app:loto}.

\begin{figure}[t]
    \centering
    \includegraphics[width=\linewidth]{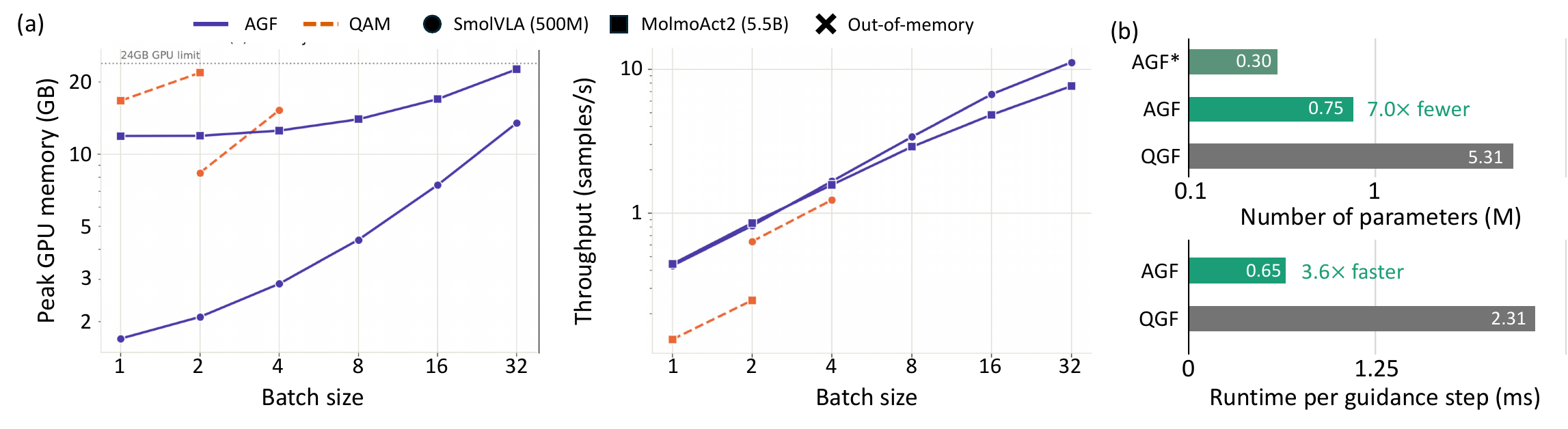}
    \vspace{-0.4cm}
    \caption{\textbf{Training and inference efficiency.}
    {(a) Peak training memory and throughput of AGF versus QAM as the batch size grows, on one RTX 4090. (b) Parameters and runtime per guidance step of AGF versus vectorized QGF (SmolVLA). AGF$^\star$: trainable part of $g_\phi$ only.}}
    \label{fig:efficiency}
    \vspace{-0.4cm}
\end{figure}

\subsection{Training and inference efficiency}
\label{subsec:efficiency}
The results above show that AGF matches or outperforms inference-time critic-guidance methods {and remains competitive with} policy fine-tuning. 
AGF reaches this performance at a fraction of their cost: its lightweight guidance design makes both training and inference substantially cheaper than baselines.
Specifically, AGF optimizes only the guidance network while keeping the pretrained VLA frozen. 
As a result, it maintains no parameter gradients or optimizer states for the VLA, substantially reducing the training cost.
Figure~\ref{fig:efficiency}(a) compares the peak GPU memory and training throughput of AGF and QAM as the batch size increases. 
Across VLA backbones, AGF scales to substantially larger batch sizes under the same memory budget. Its throughput also continues to increase with the batch size. 
\begin{wrapfigure}{r}{0.4\columnwidth}
    \centering
    \vspace{-0.4cm}
    \includegraphics[width=\linewidth]{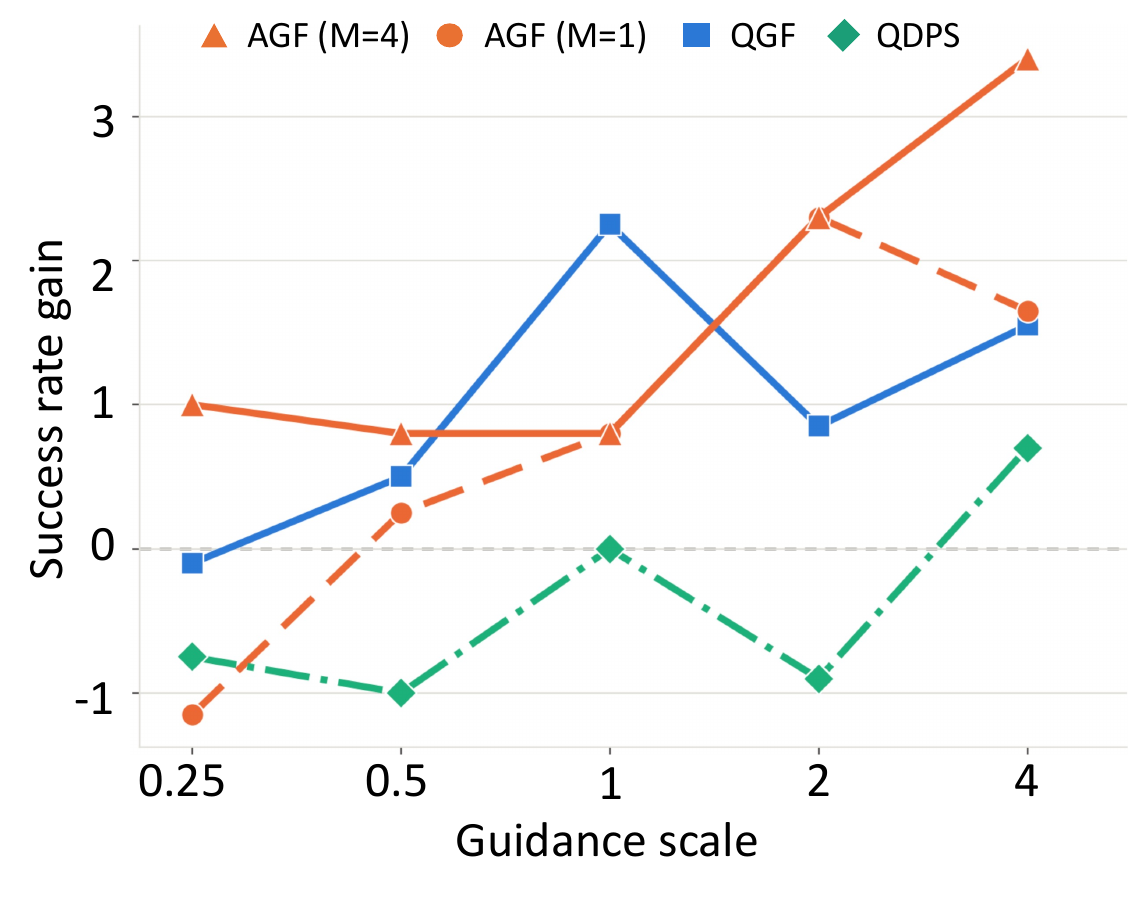}
    \caption{
    \textbf{Weight ablation.}
    Guidance-scale sweep ($0.25$--$4.0$) for each inference-time method, averaged over all LIBERO suites; AGF consistently improves performance across all scales.
    }
    \label{fig:weight}
    \vspace{-0.3cm}
\end{wrapfigure}
However, the scaling of QAM is limited earlier. This difference becomes bigger for the larger VLA backbone. 
Even under parameter-efficient fine-tuning, QAM must back-propagate through both the flow policy and the VLM, storing intermediate activations for the full backbone, whereas AGF only requires action-space vector--Jacobian products of the action expert, never entering the VLM backbone.

At inference time, QDPS and QGF evaluate and back-propagate through the critic ensemble at every flow step, whereas AGF requires only a forward pass through the guidance network. As shown in Figure~\ref{fig:efficiency}(b), AGF runs $3.6\times$ faster per guidance step than vectorized QGF and accesses $7.0\times$ fewer {parameters}; end-to-end per-chunk overhead measured inside a real control loop is reported in Section~{\ref{sec:real_robot}}. 
Thus, AGF transfers the expensive critic-gradient computation to a lightweight guidance-network training stage, enabling guidance at deployment {without the critic ensemble}.

\subsection{Effective range of guidance strength}
\label{subsec:ablation}
We now examine the effect of the guidance strength, sweeping it from $0.25$ to $4.0$ for both calibration settings, relative to the magnitude of the flow velocity predicted by the pretrained VLA. 
Figure~\ref{fig:weight} shows the average success-rate gain over the pretrained policy across all LIBERO suites; model parameters are fixed within each method, and only the inference-time guidance strength is varied.
QDPS shows no consistent gain and can degrade performance, and QGF peaks at an intermediate scale but fluctuates across the range, whereas AGF (M$=$4) attains positive gains throughout the range, performing best at both the smallest and largest scales.

\begin{figure}
    \centering
    \includegraphics[width=\linewidth]{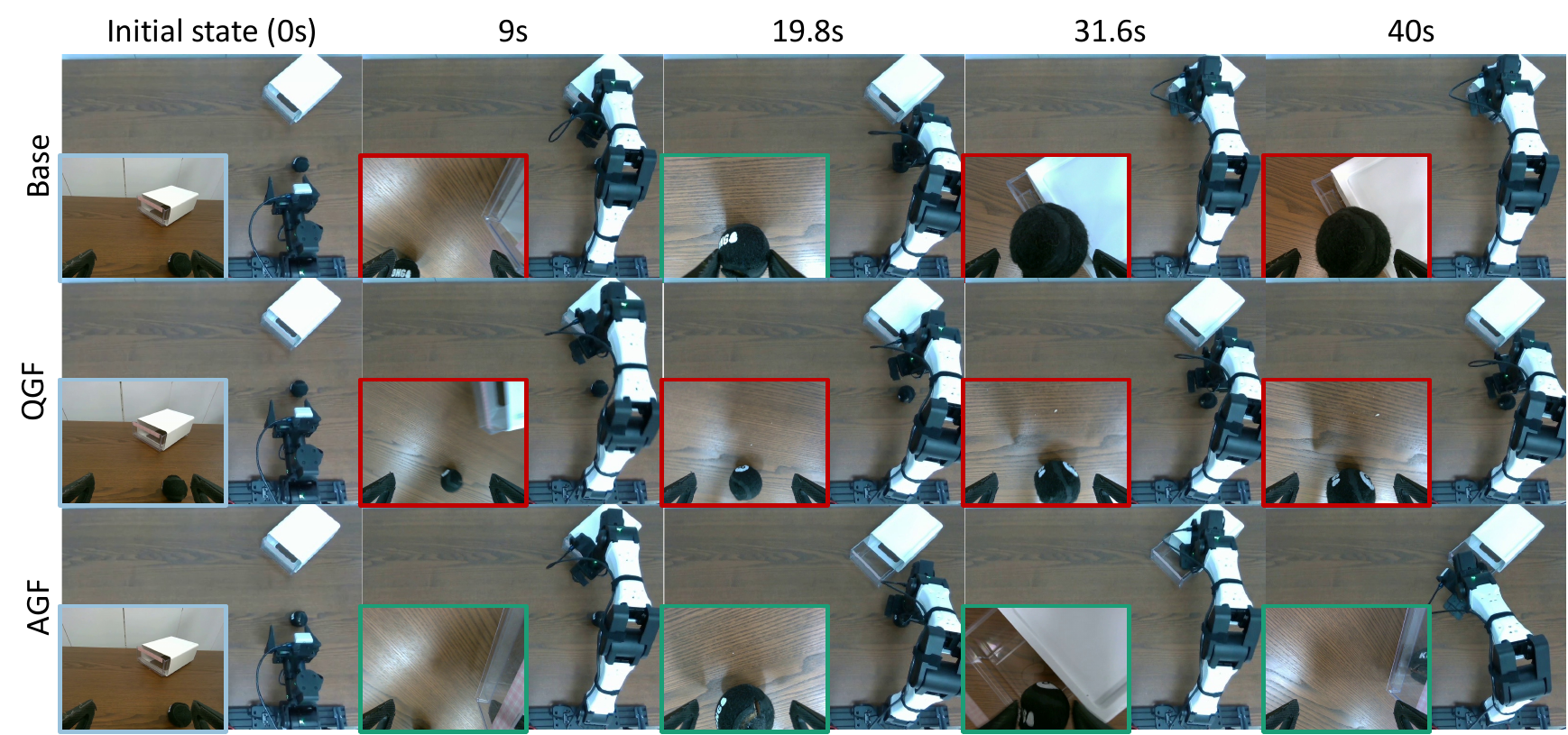}
    \vspace{-0.4cm}
    \caption{Real-robot evaluation on \textit{open-pnp-close}. Columns show top-view observations at identical timestamps across methods, each chosen where AGF completes a subtask. Insets show the right wrist camera; green and red borders mark subtask success and failure.}
    \label{fig:real_qual}
    \vspace{-0.2cm}
\end{figure}

\subsection{Real Robot Experiment}
\label{sec:real_robot}

We verify the proposed AGF pipeline, including critic training, guidance network training, and deployment {without the critic ensemble}, on the real hardware setup. In particular, we examine whether AGF improves the base VLA effectively without carrying the critic ensemble and backpropagating it at deployment.

\textbf{Setup.} We adopt the real-robot setup of MolmoAct2~\citep{fang2026molmoact2}, a YAM 6-DoF arm with a parallel-jaw gripper and two RGB cameras, and evaluate on two manipulation tasks, one short (``Pick up the ball and place it on the plate.") and one long task (``Open the box, put the ball inside, and close it"). Critics and guidance networks are trained on robot rollouts following the same procedure as in simulation. 
Guidance strengths are likewise transferred from simulation without on-robot tuning: each method uses the strength that maximizes its average gain in the simulation sweep (Figure~\ref{fig:weight}; $w=1$ for QGF, $w=4$ for AGF). 
We compare Base, QGF, and AGF over 50 paired episodes for a short task, and 25 paired episodes for a long task; the full protocol is in Appendix~\ref{app:real}.

\begin{wraptable}{r}{0.55\linewidth}
\centering
\vspace{-0.4cm}
\caption{
\textbf{Real robot results.}
Success rate (\%) over paired episodes; parentheses count discordant pairs against Base (F$\rightarrow$S, S$\rightarrow$F), on which the pooled McNemar $p$ is computed. Overhead is per action chunk.
}
\label{tab:real_robot}
\scriptsize
\setlength{\tabcolsep}{3.5pt}
\resizebox{\linewidth}{!}{
\begin{tabular}{lccc}
\toprule
& Base & QGF & AGF \\
\midrule
\textit{pnp-plate} ($n{=}50$) & 78.0 & 78.0 (4/4) & \textbf{88.0} (7/2) \\
\textit{open-pnp-close} ($n{=}25$) & 44.0 & 48.0 (5/4) & \textbf{76.0} (10/2) \\
Pooled ($n{=}75$) & 66.7 & 68.0 (9/8) & \textbf{84.0} (17/4) \\
McNemar $p$ (pooled) & -- & 1.000 & \textbf{0.007} \\
\midrule
Overhead (ms/chunk) & -- & 31.2 & \textbf{11.2} \\
\bottomrule
\end{tabular}
}
\vspace{-7pt}
\end{wraptable}
\textbf{Results.} Table~\ref{tab:real_robot} reports success rates and per-chunk guidance overhead measured during the rollouts on the same RTX 4090.
Both methods are deployed under the same no-tuning protocol: each transfers its simulation-selected strength without any on-robot calibration.
Under this protocol, AGF significantly improves the base VLA without carrying the critic ensemble at deployment, whereas QGF's transferred strength yields no measurable gain, consistent with its fluctuating response to the guidance strength in simulation (Figure~\ref{fig:weight}).
The paired gap between the two is itself significant ($+16.0$pp in AGF's favor, $p{=}0.004$): live critic guidance would require on-robot strength calibration to be effective, which is precisely the per-deployment cost that inference-time guidance is meant to avoid, whereas AGF transfers directly.
The guidance overhead of AGF is also modest ($1.12$ vs.\ $3.12$ms per flow step for vectorized QGF) relative to the ${\sim}640$ms chunk generation.
{In a representative \textit{open-pnp-close} rollout (Figure~\ref{fig:real_qual}), only AGF completes all three stages, opening the box, placing the ball inside, and closing it, whereas Base and QGF fail to complete the task. Additional rollouts are provided in Appendix~\ref{app:robot_qual}.}

\section{Conclusion}
We propose AGF, {which formulates critic guidance for a frozen flow policy as deterministic optimal control and regresses a lightweight network onto the resulting costate, the exact value gradient of the executed action.}
{On LIBERO, RoboCasa, and LIBERO-Pro, AGF improves every pretrained VLA, matches live critic guidance and policy fine-tuning, keeps its gains under a single deployed strength, and carries a simulation-chosen strength to a real robot, while removing the critic ensemble and all back-propagation from the control loop.} We discuss limitations and future directions in Appendix~\ref{app:limitations}.

\section*{Use of Large Language Models}
We used a large language model as a writing and statistical analysis aid during the preparation of this paper.
Specifically, it was used for sentence-level editing and grammar checking of author-written text, and assisting the statistical analysis of experimental results, such as the paired significance tests reported in Appendix~\ref{app:stats}. 
Other aspects, including research ideas, methodologies, experimental design and experiments, and scientific claims are made by the authors. LLM-assisted text and analyses were reviewed and verified by the authors.

\section*{Ethics Statement}

This work does not involve human subjects or private data.
Real robot experiments were conducted by the authors on a tabletop manipulation platform handling benign objects such as a ball, a plate, and a box. 
Demonstration data were collected by the authors via teleoperation and contains no personally indentifiable information.
All simulation benchmarks and pretrained checkpoints are publicly available and used under their respective licenses.

\section*{Reproducibility Statement}
We provide the complete derivations of Proposition~\ref{prop:optimal} and Lemma~\ref{prop:particle} in
Appendix~\ref{app:prop1} and Appendix~\ref{app:prop4}. Appendix~\ref{app:impl} specifies everything needed to reproduce our pipeline, from the exact pretrained checkpoints with their Hugging Face identifiers (\ref{app:backbone}), critic and guidance-network architectures, training objectives, hyperparameters, and per-setting action-chunking configurations (\ref{app:simul}, Table~\ref{tab:horizon}), baseline implementations (\ref{app:baselines}), and the full real-robot protocol—dataset statistics, fine-tuning configuration, asynchronous execution, critic-input padding, and the paired evaluation procedure (\ref{app:real}). All comparisons use identical episode seeds; the statistical procedures and selected guidance weights are detailed in Appendix~\ref{app:stats} and Appendix \ref{app:weights}. The code with trained critics and guidance networks will be publicly available upon publication.

\bibliography{iclr2027_conference}
\bibliographystyle{iclr2027_conference}

\newpage
\appendix

\section{Proofs}
\subsection{Optimal control}
\label{app:prop1}

\optimalcontrol*
\begin{proof}
For the derivation, we first consider the equivalent forward generation time $\tau\in[0,1]$, where $\tau=0$ corresponds to noise and $\tau=1$ to the clean action.
Let the controlled ODE be
\begin{equation}
    \frac{d\va_\tau}{d\tau} = \vf_\theta(\va_\tau, \vs, \tau) + \vu_\tau,
\end{equation}
with the objective
\begin{equation}
    \max_{\vu}  Q^\pi(\vs,\va_1) - \int_0^1 \frac{1}{2\beta_\tau} \|\vu_\tau\|_2^2 d\tau.
\end{equation}
The corresponding Hamiltonian is
\begin{equation}
    \mathcal{H} = \vlamb_\tau^\top \left( \vf_\theta(\va_\tau, \vs, \tau) + \vu_\tau \right) - \frac{1}{2\beta_\tau} \|\vu_\tau\|_2^2.
\end{equation}
By Pontryagin's maximum principle{~\citep{pontryagin1987mathematical}}, the optimal control maximizes the Hamiltonian pointwise. Hence,
\begin{equation}
    \frac{\partial \mathcal{H}} {\partial \vu_\tau} = \vlamb_\tau - \frac{1}{\beta_\tau}\vu_\tau = 0,
\end{equation}
which gives
\begin{equation}
    \vu_\tau^\star = \beta_\tau \vlamb_\tau.
\end{equation}
The costate follows
\begin{equation}
    \frac{d\vlamb_\tau}{d\tau}= -\frac{\partial\mathcal{H}}{\partial\va_\tau}=-\left(\frac{\partial\vf_\theta(\va_\tau, \vs, \tau)}{\partial\va_\tau}\right)^\top \vlamb_\tau,
\end{equation}
with the terminal condition $\vlamb_1 = \nabla_{\va_1}Q^\pi(\vs,\va_1)$.


Our flow convention instead uses $t = 1$ for noise and $t = 0$ for the clean action. Under $t = 1-\tau$, the drift and control map to $\vf_\theta = -\vv_\theta$ and $\vu_\tau = -\vu_t$, and the two sign flips in the costate dynamics cancel, giving
\begin{equation}
    \vu_t^\star=-\beta_t\vlamb_t, \qquad \frac{d\vlamb_t}{dt}= -\left(\frac{\partial\vv_\theta(\va_t, \vs, t)}{\partial\va_t}\right)^\top \vlamb_t,
\end{equation}
with $\vlamb_0=\nabla_{\va_0}Q^\pi(\vs,\va_0)$, which proves Proposition~\ref{prop:optimal}.
\end{proof}

{\subsection{Particle-averaged adjoint}\label{app:prop4}
\begin{lemma}\label{prop:particle}
Fix a controlled trajectory $\{\va_{t_j}\}_{j=0}^{n}$ with $t_0=0$ and any $n\in\{1,\dots,N\}$, and let $\widehat{\vlamb}$ follows~\eqref{eqn:particle_adjoint}, where the particle sets $\{\veps_m^{(j)}\}_{m=1}^{M}$ are drawn i.i.d.\ from $\mathcal{N}(\bm{0},\rmI)$, independently across flow steps $j$ and of $(\vs,\va_1)$. Let $\bar J_{t_j}=\mathbb{E}_{\veps}[J_{t_j}(\va_{t_j}+\sigma\veps)]$ and assume $\mathbb{E}_{\veps}\|J_{t_j}(\va_{t_j}+\sigma\veps)-\bar J_{t_j}\|_F^2\le c$ for all $j$. Then, for every $M\ge1$,
\begin{equation}
    \mathbb{E}\big[\widehat{\vlamb}_{t_n}\big]=\bar{\vlamb}_{t_n}:=\big(\rmI+\Delta t\bar J_{t_n}\big)^{\top}\cdots\big(\rmI+\Delta t\bar J_{t_1}\big)^{\top}\nabla_{\va_0}Q^\pi(\vs,\va_0),
\label{eqn:particle_mean}
\end{equation}
\begin{equation}
    \mathbb{E}\big\|\widehat{\vlamb}_{t_n}-\bar{\vlamb}_{t_n}\big\|_2^2\le\frac{C_n}{M},
\label{eqn:particle_var}
\end{equation}
where $C_n$ does not depend on $M$. Consequently, with the rollout guidance fixed, $t$ uniform on the grid, and a finite loss, the minimizer of~\eqref{eqn:agf_loss} with target $\widehat{\vlamb}$ is $\bar{\vlamb}$ a.e.\ for every $M$.
\end{lemma}
\begin{proof}
Let $\vlamb_0=\nabla_{\va_0}Q^\pi(\vs,\va_0)$ and
\begin{equation}
    \mA_j=\rmI+\Delta t\widehat{J}_{t_j}(\va_{t_j}),\qquad
    \bar{\mA}_j=\rmI+\Delta t\bar J_{t_j},\qquad
    \mE_j=\widehat{J}_{t_j}(\va_{t_j})-\bar J_{t_j}.
\end{equation}
Unrolling \eqref{eqn:particle_adjoint} from $\widehat{\vlamb}_{t_0}=\vlamb_0$:
\begin{equation}
    \widehat{\vlamb}_{t_n}=\mA_n^\top\mA_{n-1}^\top\cdots\mA_1^\top\vlamb_0,\qquad \mA_j=\bar{\mA}_j+\Delta t\mE_j .
\label{eqn:particle_unroll}
\end{equation}
The particles enter only $\widehat{J}$, not the forward pass, so $\va_{t_j}$ and $\vlamb_0$ are fixed, $\mA_1,\dots,\mA_n$ are independent, and by~\eqref{eqn:particle_jacobian}:
\begin{equation}
    \mathbb{E}[\mE_j]=\frac{1}{M}\sum_{m=1}^{M}\mathbb{E}_{\veps}\big[J_{t_j}(\va_{t_j}+\sigma\veps_m^{(j)})\big]-\bar J_{t_j}=\bm{0},\qquad \mathbb{E}[\mA_j]=\bar{\mA}_j .
\end{equation}

\textbf{Mean.} Taking expectations in~\eqref{eqn:particle_unroll}:
\begin{align*}
    \mathbb{E}\big[\widehat{\vlamb}_{t_n}\big]
    &=\mathbb{E}[\mA_n]^\top\cdots\mathbb{E}[\mA_1]^\top\vlamb_0\quad\text{(independence)}\\
    &=\bar{\mA}_n^\top\cdots\bar{\mA}_1^\top\vlamb_0=\bar{\vlamb}_{t_n}.
\end{align*}

\textbf{Variance.} Each $\mE_j$ is an average of $M$ i.i.d.\ zero-mean matrices:
\begin{align*}
    \mathbb{E}\|\mE_j\|_2^2
    &\le\mathbb{E}\|\mE_j\|_F^2\quad(\|\cdot\|_2\le\|\cdot\|_F)\\
    &=\frac{1}{M}\mathbb{E}_{\veps}\big\|J_{t_j}(\va_{t_j}+\sigma\veps)-\bar J_{t_j}\big\|_F^2\le\frac{c}{M}\quad\text{(i.i.d., zero mean)}.
\end{align*}
Expanding~\eqref{eqn:particle_unroll} by the set $\mathcal{P}\subseteq\{1,\dots,n\}$ of factors that contribute $\Delta t\mE_j$:
\begin{equation}
    \widehat{\vlamb}_{t_n}-\bar{\vlamb}_{t_n}=\sum_{\mathcal{P}\neq\emptyset}\vd_{\mathcal{P}},\qquad
    \vd_{\mathcal{P}}=(\mB^{\mathcal{P}}_n)^\top\cdots(\mB^{\mathcal{P}}_1)^\top\vlamb_0,\qquad
    \mB^{\mathcal{P}}_j=\begin{cases}\Delta t\mE_j, & j\in \mathcal{P},\\ \bar{\mA}_j, & j\notin \mathcal{P}.\end{cases}
\end{equation}
For $\mathcal{P}\neq\mathcal{P}'$, pick $j$ in exactly one of them; then $\vd_{\mathcal{P}}^\top\vd_{\mathcal{P}'}$ is linear in $\mE_j$:
\begin{equation}
    \mathbb{E}\big[\vd_{\mathcal{P}}^\top\vd_{\mathcal{P}'}\big]=\mathbb{E}\Big[\mathbb{E}\big[\vd_{\mathcal{P}}^\top\vd_{\mathcal{P}'}\big|\{\mE_i\}_{i\neq j}\big]\Big]=0 .
\end{equation}
With $\rho=\max_j\|\bar{\mA}_j\|_2$, $\ell=\|\vlamb_0\|_2$, and $\eta=\Delta t^2c/M$:
\begin{align*}
    \mathbb{E}\big\|\widehat{\vlamb}_{t_n}-\bar{\vlamb}_{t_n}\big\|_2^2
    &=\textstyle\sum_{\mathcal{P}\neq\emptyset}\mathbb{E}\|\vd_{\mathcal{P}}\|_2^2\quad\text{(cross terms vanish)}\\
    &\le\ell^2\textstyle\sum_{\mathcal{P}\neq\emptyset}\prod_{j\notin \mathcal{P}}\|\bar{\mA}_j\|_2^2\prod_{j\in \mathcal{P}}\Delta t^2\mathbb{E}\|\mE_j\|_2^2\quad\text{(submultiplicativity, indep.)}\\
    &\le\ell^2\textstyle\sum_{\mathcal{P}\neq\emptyset}\rho^{2(n-|\mathcal{P}|)}\eta^{|\mathcal{P}|}=\ell^2\big[(\rho^2+\eta)^n-\rho^{2n}\big]\quad\text{(binomial theorem)}\\
    &\le n\ell^2\eta(\rho^2+\eta)^{n-1}\quad\text{(mean value theorem)}\\
    &\le\frac{C_n}{M}\quad(M\ge1),
\end{align*}
with $C_n=n\ell^2\Delta t^2c(\rho^2+\Delta t^2c)^{n-1}$.

\smallskip
\textbf{Minimizer.} For a fixed input $(\vs,\va_t,t)$ and any output $\vz$:
\begin{equation}
    \mathbb{E}\big[\|\vz-\widehat{\vlamb}_t\|_2^2\big|\vs,\va_t,t\big]=\big\|\vz-\mathbb{E}[\widehat{\vlamb}_t\mid\vs,\va_t,t]\big\|_2^2+\mathrm{const}.
\end{equation}
The input fixes the trajectory from $\va_t$ to $\va_0$, and the particles are independent of it, so by~\eqref{eqn:particle_mean}:
\begin{equation}
    g_\phi^\star(\vs,\va_t,t)=\mathbb{E}\big[\widehat{\vlamb}_t\mid\vs,\va_t,t\big]=\bar{\vlamb}_t\quad\text{for every } M.
\end{equation}
\end{proof}

\begin{remark}
$\bar{\vlamb}$ propagates the smoothed Jacobian along the unsmoothed trajectory, so it differs from $\vlamb$ by a smoothing bias that vanishes as $\sigma\to0$ under suitable continuity and integrability conditions. Independent particle draws across flow steps ensure unbiasedness; reusing particles across steps can introduce bias. With rollout inputs held fixed and stop-gradient applied to the targets, the squared-loss gradient is affine in its target. Thus, the stochastic regression gradient is unbiased relative to regression against $\bar{\vlamb}$, for any $g_\phi$ and any fixed terminal vector, including the ensemble critic gradient.
\end{remark}
}

\section{Implementation Details}
\label{app:impl}

\subsection{VLA Backbones}
\label{app:backbone}
All experiments build on publicly released checkpoints from the Hugging Face
hub, each fine-tuned on the demonstrations of the corresponding benchmark. We
use all backbones frozen; AGF trains only the critic and the guidance network
on top.

\textbf{SmolVLA}~\citep{shukor1844smolvla}, trained on LIBERO (\texttt{HuggingFaceVLA/smolvla\_libero}), used for LIBERO.

$\mathbf{\pi_{0.5}}$~\citep{intelligence2025pi_}, trained on RoboCasa (\texttt{lerobot/pi05\_robocasa}), used for the RoboCasa atomic tasks.

\textbf{MolmoAct2}~\citep{fang2026molmoact2}, trained on LIBERO; we use the LeRobot-format conversion (\texttt{allenai/MolmoAct2-LIBERO-LeRobot}) for LIBERO-Pro. For the real-robot experiments (Section~4.4), we instead start from the bimanual-YAM checkpoint (\texttt{allenai/MolmoAct2-BimanualYAM}) and fine-tune it on our demonstrations (Appendix~\ref{app:real}).

\begin{table}[h]
\centering
\caption{Action-chunking configuration per setting. The critic and the
guidance network share the same horizon. $\dagger$ For the details, see Appendix \ref{app:real}.}
\label{tab:horizon}
\begin{tabular}{lccc}
\toprule
Setting & Chunk length & Executed steps & Critic/guidance horizon \\
\midrule
SmolVLA + LIBERO        & 50 & 10 & 10 \\
$\pi_{0.5}$ + RoboCasa  & 50 & 50 & 50 \\
MolmoAct2 + LIBERO-Pro  & 10 & 10 & 10 \\
MolmoAct2 + real robot $^\dagger$   & 30 & $\sim$15 (async) & 29\\
\bottomrule
\end{tabular}
\end{table}

\subsection{Simulation benchmarks}
\label{app:simul}

\subsubsection{Critics}
We adopt the critic architecture from QPILOTS~\citep{ruan2026qpilots}: an ensemble of 10 Q-functions, each a four-layer CNN encoder followed by an MLP, aggregated pessimistically as $\bar{Q} = \mathrm{mean}(Q_j) - 0.5\mathrm{std}(Q_j)$.
We train the critic for 10K iterations using the objective in \eqref{eqn:sarsa_td} with $\gamma=0.99$, a delayed target critic updated by EMA with decay $0.995$, and Adam with a learning rate of $3\times10^{-4}$.
The VLA predicts an action chunk $\va_0$ but executes only a prefix of it before replanning, whose length is model-specific (Table~\ref{tab:horizon}). The critic's action input is this executed prefix rather than the full predicted chunk.

\subsubsection{AGF}
The guidance network is trained on the same rollout dataset used for critic training for 5k rollout updates for SmolVLA and $\pi_{0.5}$, and 1k updates for MolmoAct2 using AdamW with a learning rate of $3\times10^{-4}$, with intermediate flow states generated by the frozen VLA policy under the current guidance network and adjoint targets computed with the frozen critic.
Note that training the guidance network requires no additional rollouts since it reuses the rollout data collected for critic training. Therefore, the entire AGF pipeline after data collection runs offline without further robot interaction.
For particle-based adjoint estimation, we use $M=4$ particles with $\sigma=0.02$.
RoboCasa evaluation uses the default task-specific episode horizons.
We set $\beta_t = 1$ during training, so that controlled trajectories are generated with $\vu_t = -g_\phi(\vs, \va_t, t)$, and absorb the guidance scale into the inference-time weight $w$, so that the deployed control is $\vu_t = -w g_\phi(\vs, \va_t, t)$.

\subsubsection{Baselines}
\label{app:baselines}

All inference-time baselines share the same pretrained VLA, the same task-specific critic ensemble, and the same number of action-generation steps as AGF; they differ only in how the critic signal is injected during generation.

\textbf{Q-BoN (Best-of-$N$).} Q-BoN applies the critic only at the terminal step. For each environment step, we sample $N$ action chunks $\{\va_0^{(i)}\}_{i=1}^N$ from the unguided pretrained flow with independent initial noises, evaluate the aggregated critic $\bar{Q}(\vs, \va_0^{(i)})$ for each candidate, and execute the chunk with the highest value. Q-BoN requires no back-propagation and uses intermediate generation as is, but its per-step cost scales linearly with $N$. Analogously to the guidance-weight sweep of the gradient-based methods, we sweep $N \in \{4, 8, 16\}$ and select the best value per suite; the selected values are listed in Table~\ref{tab:guidance_weights}.

\textbf{QDPS.} QDPS adapts diffusion posterior sampling~\citep{chung2023diffusion} to critic guidance. Specifically, at each flow step, the critic is evaluated at the Tweedie estimate $\hat{\va}_{0|t}$ and its gradient is back-propagated through the flow model to the intermediate action, $\vg_t = \nabla_{\va_t} \bar{Q}(\vs, \hat{\va}_{0|t})$, which is then added to the velocity with weight $w$. Unlike QGF, QDPS retains the Jacobian of the Tweedie map, requiring a backward pass through $\vv_\theta$ at every step, which is computationally more expensive.

\textbf{QGF.} QGF~\citep{zhou2026test} further approximates the Jacobian of the Tweedie map by the identity ({\eqref{eqn:qgf}}), evaluating $\vg_t = \partial \bar{Q}(\vs, \hat{\va}_{0|t}) / \partial \hat{\va}_{0|t}$ and injecting it directly (Algorithm~{\ref{alg:qgf_inf}}). This removes the backward pass through the flow but back-propagates through the critic ensemble at every step. As noted in Section~{\ref{sec:experiments}}, we treat QPILOTS~\citep{ruan2026qpilots} and GAF~\citep{yang2026guided} as variants of QGF sharing the same core update rule.

\textbf{QAM.} QAM~\citep{li2026qlearning} amortizes the critic signal into the policy itself: the flow policy is fine-tuned with adjoint-matching supervision so that no critic is needed at inference. We fine-tune the policy with LoRA~\citep{hu2022lora} (rank 16) using AdamW with learning rate $1\times10^{-4}$, weight decay $1\times10^{-4}$, and gradient clipping at norm $1.0$, for 4{,}000 steps under the same wall-clock budget as AGF training. Following the original method, the adjoint supervision is clipped element-wise at magnitude $1.0$ for numerical stability; for stable adaptation of the large pretrained VLA, we additionally warm up training with a behavior-cloning mixture ($\lambda_{\max}=0.5$ over the first 200 steps). Since the guidance strength is absorbed into the fine-tuned weights, QAM has no inference-time hyperparameter and reports a single result across both calibration settings.

\begin{algorithm}[h]
\caption{\prev{Training procedure of AGF}}
\label{alg:agf}
\begin{algorithmic}[1]
\Require frozen flow $\vv_\theta$, critic $Q_\omega$, guidance network $g_\phi$, guidance strength $\beta_t$, time grid $\{t_n\}_{n=0}^{N}$ with $t_0=0$ (clean) and $t_N=1$ (noise), particles $M$, perturbation scale $\sigma$
\For{each training iteration}
    \State sample $\vs$; generate $\{\va_{t_n}\}$ by integrating $d\va_t=[\vv_\theta(\va_t,\vs,t)-\beta_t g_\phi(\vs,\va_t,t)]dt$
    \State $\vlamb \gets \nabla_{\va_0} Q_\omega(\vs, \va_0)$ \Comment{terminal seed at the clean boundary}
    \For{$n = 1, \dots, N$} \Comment{clean $\to$ noise}
        \State $\widehat{J} \gets \frac{1}{M}\sum_{m=1}^{M} \partial \vv_\theta(\va_{t_n}+\sigma\veps_m,\vs,t_n)/\partial\va$, \quad $\veps_m \sim \mathcal{N}(0,\rmI)$ 
        \State $\vlamb \gets (\rmI + \Delta t \widehat{J})^\top \vlamb$; \quad store $(\va_{t_n}, \vlamb)$ \Comment{VJP through $\vv_\theta$}
    \EndFor
    \State update $\phi$ with $\mathcal{L}_{\mathrm{AGF}}(\phi)=\sum_n \big\| g_\phi(\vs,\va_{t_n},t_n) - \operatorname{sg}[\vlamb_{t_n}]\big\|_2^2$
\EndFor
\end{algorithmic}
\end{algorithm}

\noindent
\begin{minipage}[t]{0.49\linewidth}
\vspace{0pt}
\begin{algorithm}[H]
\caption{\prev{QGF guidance (inference)}}
\label{alg:qgf_inf}
\begin{algorithmic}[1]
\Require frozen flow $\vv_\theta$, critic $Q_\omega$, guidance weight $w$, step $\Delta t < 0$
\State $\va_1 \sim \mathcal{N}(0,\rmI)$
\For{$t = 1, 1+\Delta t, \dots$ until $t=0$}
    \State $\hat\va_{0|t} \gets \va_t - t\vv_\theta(\va_t,\vs,t)$
    \State $\vg \gets \nabla_{\hat\va_{0|t}} Q_\omega(\vs,\hat\va_{0|t})$ \Comment{back-prop}
    \State $\va_{t+\Delta t} \gets \va_t + \Delta t\left[\vv_\theta(\va_t,\vs,t) - w\vg\right]$
\EndFor
\State \Return $\va_0$
\end{algorithmic}
\end{algorithm}
\end{minipage}\hfill
\begin{minipage}[t]{0.49\linewidth}
\vspace{0pt}
\begin{algorithm}[H]
\caption{\prev{AGF guidance (inference, ours)}}
\label{alg:agf_inf}
\begin{algorithmic}[1]
\Require frozen flow $\vv_\theta$, guidance network $g_\phi$, guidance weight $w$, step $\Delta t < 0$
\State $\va_1 \sim \mathcal{N}(0,\rmI)$
\For{$t = 1, 1+\Delta t, \dots$ until $t=0$}
    \State $\vg \gets g_\phi(\vs,\va_t,t)$ \Comment{single forward pass}
    \State $\va_{t+\Delta t} \gets \va_t + \Delta t\left[\vv_\theta(\va_t,\vs,t) - w\vg\right]$
\EndFor
\State \Return $\va_0$
\end{algorithmic}
\end{algorithm}
\end{minipage}

\newpage
\subsection{Real Robot Experiments}
\label{app:real}

\begin{wrapfigure}{r}{0.48\textwidth}
  \vspace{-\intextsep}
  \centering
  \includegraphics[width=0.46\textwidth]{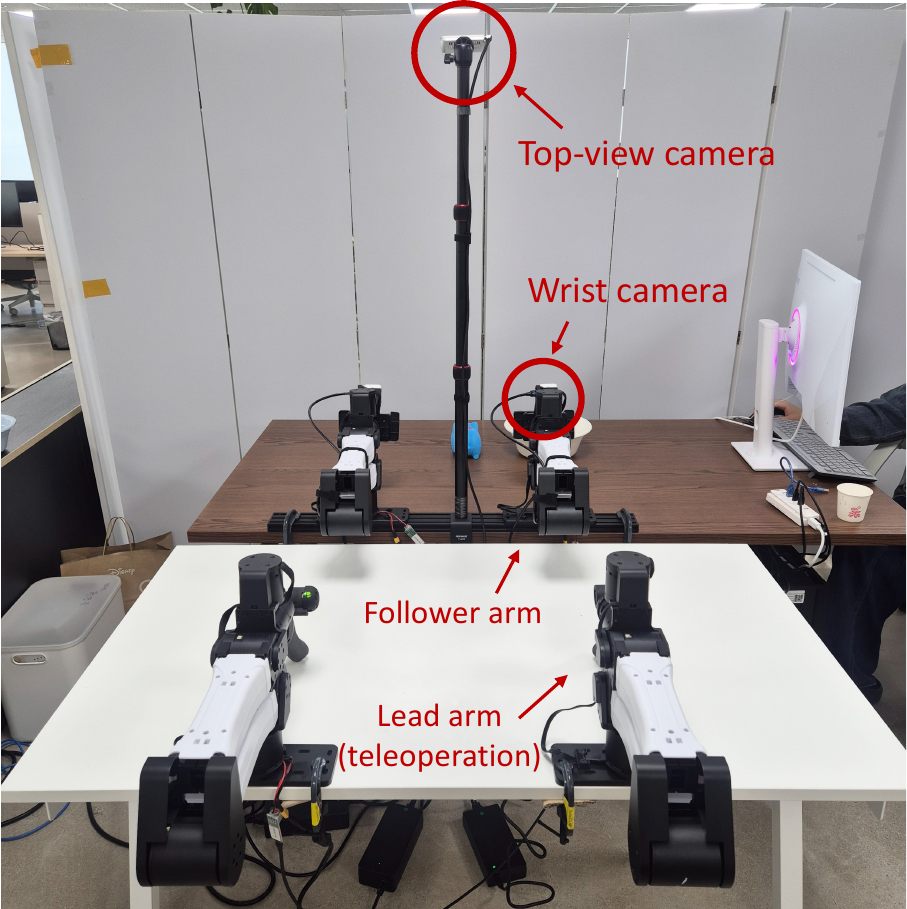}
  \caption{Real-robot setup. Bimanual platform with the top-view camera; the platform is used in a single-arm configuration and only one
  follower arm is controlled. Wrist camera on that arm.}
  \label{fig:setup}
  \vspace{-\intextsep}
\end{wrapfigure}

\textbf{Robot setup.}
Apart from the task set, we follow the real-robot setup of MolmoAct2~\citep{fang2026molmoact2}: a YAM 6-DoF robot arm equipped with a parallel-jaw gripper, resulting in a 7-dimensional action space (six joint positions and one gripper command). The platform is bimanual, but we use it in a single-arm configuration and control only the right arm throughout data collection and evaluation.
The policy observes two RGB camera views (top and right) together with the robot's proprioceptive state, and receives a language instruction describing the task.
Raw joint states and teleoperation commands are recorded at approximately 200Hz and linearly interpolated onto the 30Hz camera timestamps, so all training and evaluation data are at 30Hz. 
At deployment, the policy runs in a 30Hz control loop and predicts 30-step action chunks (1s). Chunks are generated asynchronously: a new chunk is requested every 0.5s and, once available, replaces the remaining action queue, so roughly 15 steps of each chunk are executed under nominal inference latency.
All policies run on a single NVIDIA RTX~4090, and all latencies are measured on the same machine.

\textbf{Dataset and training.}
We collect approximately 100 episodes for each of four tasks (eight data-collection variants differing in object instances, e.g., different balls for \emph{pnp-plate}, grouped into four tasks). The exact statistics are in Table~\ref{tab:real_world_dataset}.
We fully fine-tune MolmoAct2 (starting from \texttt{allenai/MolmoAct2-BimanualYAM}; all parameters updated except the
frozen embeddings) on the collected demonstrations through behavior cloning, using AdamW (lr $10^{-5}$, cosine decay to $10^{-6}$ with 200-step warmup, no weight decay) with batch size 108 for 20,000 steps on a single B200 GPU and standard image augmentations (color jitter, sharpness jitter, random affine within $\pm5^\circ$ and $\pm5\%$ translation); we use the checkpoint at 14,000 steps.
Following the per-task protocol of the simulated experiments, critics and guidance networks are trained for the two evaluation tasks, using rollouts of the fine-tuned policy on each task.
The critic takes the full 29-step action chunk as input (Table~\ref{tab:horizon})\footnote{While the policy predicts 30-step chunks, the longest executed action sequence observed in the recorded data was 29 steps, so we set the critic horizon to 29 accordingly.}. Because a new chunk replaces the action queue roughly every 0.5s under asynchronous execution, the recorded per-chunk action sequences are typically 15--20 steps rather than the full 29; we pad them to 29 steps by repeating the last predicted action. Discarding incomplete chunks instead would remove nearly all records, including the terminal ones where success and failure signals concentrate.
The two evaluation tasks are ``pick up the ball and place it on the plate'' (\textit{pnp-plate}), where a black ball must be picked and placed onto a plate, and ``open the box, put the ball inside, and close it'' (\textit{open-pnp-close}), a long-horizon sequence in which the robot opens a box, picks and places the black ball inside, and closes the box.

\begin{table}[t]
\centering
\vspace{-10pt}
\caption{Real-world dataset statistics. $^\dagger$denotes the two evaluation tasks.}
\label{tab:real_world_dataset}
\resizebox{0.5\linewidth}{!}{%
\begin{tabular}{lccc}
\toprule
\textbf{Task} & \textbf{Episodes} & \textbf{Frames} & \textbf{Avg. Length (s)} \\
\midrule
pnp-plate$^\dagger$      & 106 & 47{,}002 & 14.8 \\
pnp-box                  & 108 & 47{,}616 & 14.7 \\
open-pnp-close$^\dagger$ & 108 & 91{,}614 & 28.3 \\
open-place               & 102 & 82{,}313 & 26.9 \\
\midrule
Overall                  & 424 & 268{,}545 & 21.2 \\
\bottomrule
\end{tabular}%
}
\vspace{-8pt}
\end{table}

\textbf{Evaluation details.}
{\textit{pnp-plate} is evaluated over 50 paired episodes and \textit{open-pnp-close} over 25}, with a time limit of 30s for \textit{pnp-plate} and 45s for \textit{open-pnp-close}. 
Episodes differ in their initial configuration: the ball and plate positions vary across \textit{pnp-plate} episodes, and the ball position varies across \textit{open-pnp-close} episodes. All methods are nevertheless evaluated on the same set of initial configurations, enabling paired comparison.
We report the success rate (\%). For \textit{pnp-plate}, an episode is successful once the ball lies entirely on the plate and the gripper has opened. For \textit{open-pnp-close}, an episode is successful only if the ball is placed inside the box and the box is fully closed.
Latencies are measured with CUDA synchronization inside the deployed control loop (10 denoising steps, two $480{\times}640$ camera views), with each method at its deployed strength; QGF uses a vectorized (vmap) critic-ensemble implementation.

\textbf{Paired evaluation.}
We built a dedicated interface for paired evaluation across the compared methods (Figure~\ref{fig:ref}). Panel~(A) tracks which scenarios have been evaluated. Panel~(B) shows the reference camera view for initialization, captured when the first method is evaluated; subsequent methods are initialized against this reference. Specifically, we compare the reference view~(B) with the current view~(C) and compute their difference map~(D), which ensures that each scenario is initialized consistently across all methods.

\begin{figure}[t]
    \centering
    \includegraphics[width=0.8\linewidth]{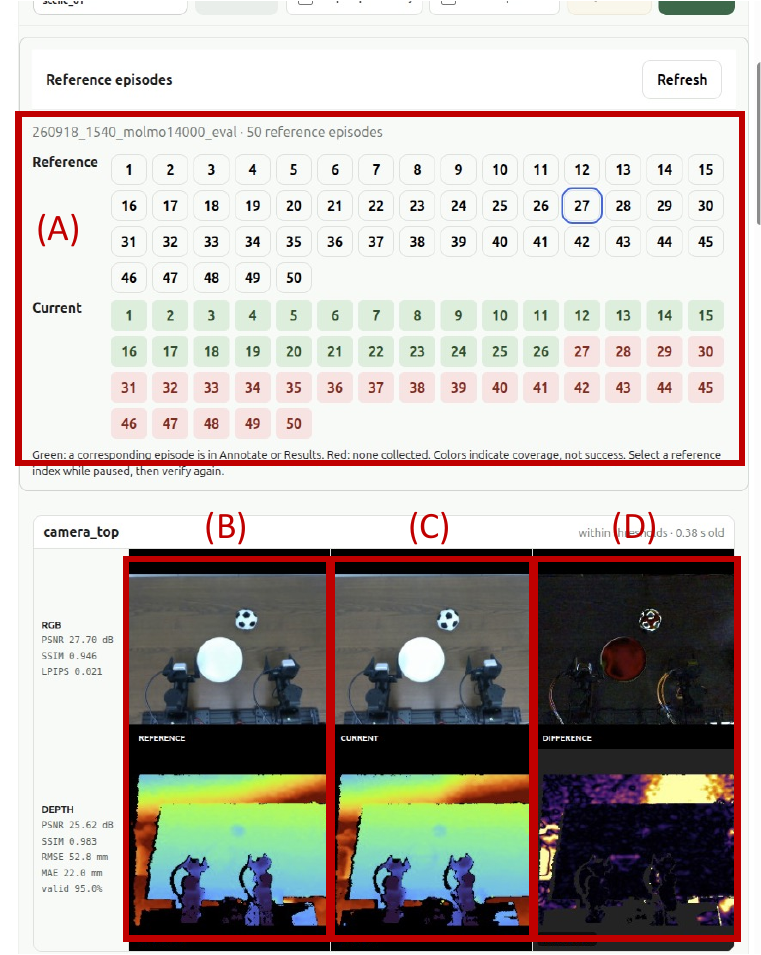}
    \caption{\textbf{Paired evaluation interface.} (A) Coverage tracker over the 50 reference episodes. (B) Reference view captured when the first method is evaluated, (C) the current scene before a subsequent method is run, and (D) their difference map, used to verify that each scenario is initialized consistently across methods.}
    \label{fig:ref}
\end{figure}

\section{Additional results and analysis}

\subsection{Per-condition LIBERO-Pro Results}
\label{app:libero_pro}
Table~\ref{tab:libero_pro} reports the per-condition LIBERO-Pro breakdown; the suite-averaged MolmoAct2 results correspond to the LIBERO-Pro rows of Table~\ref{tab:eval}.
The Clean and Semantic conditions are near saturation for MolmoAct2, leaving little room for guidance, whereas the gains of all guidance methods concentrate on the harder conditions, most notably Position, where AGF ($M=4$) improves the pretrained policy from 23.0 to 27.0 under task-level calibration.
Across all five conditions, AGF remains on par with QDPS and QGF under both calibration settings, extending the suite-level observation of Section~\ref{sec:experiments} to individual visual perturbations.
Q-BoN collapses uniformly across conditions; we analyze this failure in Appendix~\ref{app:qbon}.

\begin{table*}[t]
\centering
\caption{
\textbf{Per-condition LIBERO-Pro results with MolmoAct2.}
Task success rate (\%) averaged over the four LIBERO suites for each visual condition, under suite-level and task-level calibration; the Avg.\ column corresponds to the LIBERO-Pro rows of Table~\ref{tab:eval}.
Bold and underline mark the best and second-best inference-time method per column within each calibration setting; our method is shaded.
}
\label{tab:libero_pro}
\footnotesize
\setlength{\tabcolsep}{5pt}
\resizebox{0.9\linewidth}{!}{%
\begin{tabular}{llcccccc}
\toprule
Calibration
& Method
& Clean
& Position
& Object
& Semantic
& Environment
& Avg. \\
\midrule

\multirow{6}{*}{Suite-level}
& Pretrained VLA & 97.2 & 23.0 & 83.5 & 97.2 & 74.0 & 75.0 \\
& Q-BoN ($N{=}4$) & 8.2 & 1.2 & 2.0 & 6.5 & 3.8 & 4.3 \\
& QDPS           & \underline{97.5} & 24.0 & 84.0 & \textbf{98.2} & \textbf{76.0} & 76.0 \\
& QGF            & 97.2 & 24.2 & \textbf{85.2} & \underline{98.0} & \underline{75.8} & \underline{76.1} \\
& \cellcolor{oursbg}AGF ($M{=}1$) & \cellcolor{oursbg}97.0 & \cellcolor{oursbg}\underline{24.8} & \cellcolor{oursbg}\underline{84.8} & \cellcolor{oursbg}97.5 & \cellcolor{oursbg}\underline{75.8} & \cellcolor{oursbg}76.0 \\
& \cellcolor{oursbg}AGF ($M{=}4$) & \cellcolor{oursbg}\textbf{98.0} & \cellcolor{oursbg}\textbf{25.8} & \cellcolor{oursbg}84.0 & \cellcolor{oursbg}97.8 & \cellcolor{oursbg}75.2 & \cellcolor{oursbg}\textbf{76.2} \\
\midrule

\multirow{6}{*}{Task-level}
& Pretrained VLA & 97.2 & 23.0 & 83.5 & 97.2 & 74.0 & 75.0 \\
& Q-BoN ($N{=}4$) & 8.2 & 1.2 & 2.0 & 6.5 & 3.8 & 4.3 \\
& QDPS           & \underline{97.8} & \underline{26.8} & 85.2 & 98.0 & \textbf{79.2} & \underline{77.4} \\
& QGF            & \textbf{98.2} & \textbf{27.0} & \underline{85.5} & \textbf{98.5} & 77.8 & \underline{77.4} \\
& \cellcolor{oursbg}AGF ($M{=}1$) & \cellcolor{oursbg}97.2 & \cellcolor{oursbg}26.2 & \cellcolor{oursbg}\textbf{85.8} & \cellcolor{oursbg}98.0 & \cellcolor{oursbg}\underline{78.8} & \cellcolor{oursbg}77.2 \\
& \cellcolor{oursbg}AGF ($M{=}4$) & \cellcolor{oursbg}97.5\hspace{0pt} & \cellcolor{oursbg}\textbf{27.0} & \cellcolor{oursbg}\textbf{85.8} & \cellcolor{oursbg}\underline{98.2} & \cellcolor{oursbg}\textbf{79.2} & \cellcolor{oursbg}\textbf{77.5} \\
\bottomrule
\end{tabular}%
}
\end{table*}

\subsection{Q-BoN Over-optimization on MolmoAct2}
\label{app:qbon}

On MolmoAct2, Q-BoN collapses across all four LIBERO-Pro suites (Table~\ref{tab:eval}), although the same critic provides effective guidance for QGF and AGF. We verified that Q-BoN uses the same critic checkpoint and pessimistic aggregation, with the correct ranking polarity and independently sampled candidates. 
The failure instead arises from best-of-$N$ over-optimization. QGF and AGF make local corrections around the base-policy trajectory, where the critic is sufficiently supported by its finite rollout data. 
In contrast, maximizing over $N$ independent samples increasingly favors low-density candidates with positive value-estimation errors, resulting in actions that score highly under the critic but fail in execution.

Table~\ref{tab:qbon_diag} shows this effect across all four suites. As $N$ increases, the predicted advantage of the selected candidate, $\bar{Q}{\mathrm{best}}-\bar{Q}{\mathrm{mean}}$, grows monotonically, while the actual success rate collapses, which is a characteristic signature of over-optimization against an imperfect value estimate. Task-level selection does not resolve the issue because for every task with nonzero success, $N{=}4$ is already optimal, with one tie, while the remaining tasks fail for all evaluated values of $N$. Consequently, the suite- and task-level results in Table~\ref{tab:eval} coincide.

\begin{table}[h]
\centering
\caption{
\textbf{Best-of-$N$ over-optimization on MolmoAct2.}
Suite-level success rate (\%) and the critic's apparent advantage of the selected candidate ($\bar{Q}_{\mathrm{best}}-\bar{Q}_{\mathrm{mean}}$, averaged over episodes) as the number of candidates $N$ grows. Across all four suites, the apparent advantage increases monotonically while the actual success rate collapses.
}
\label{tab:qbon_diag}
\footnotesize
\begin{tabular}{lcccc}
\toprule
& & \multicolumn{3}{c}{$N$} \\
\cmidrule(lr){3-5}
Suite & & 4 & 8 & 16 \\
\midrule
\multirow{2}{*}{Goal}      & Success rate & 10.4 & 2.0 & 0.6 \\
                           & $\bar{Q}_{\mathrm{best}}-\bar{Q}_{\mathrm{mean}}$ & 0.224 & 0.344 & 0.446 \\
\midrule
\multirow{2}{*}{Object}    & Success rate & 1.0 & 0.0 & 0.0 \\
                           & $\bar{Q}_{\mathrm{best}}-\bar{Q}_{\mathrm{mean}}$ & 0.218 & 0.309 & 0.359 \\
\midrule
\multirow{2}{*}{Spatial}   & Success rate & 5.8 & 1.2 & 0.2 \\
                           & $\bar{Q}_{\mathrm{best}}-\bar{Q}_{\mathrm{mean}}$ & 0.253 & 0.363 & 0.438 \\
\midrule
\multirow{2}{*}{LIBERO-10} & Success rate & 0.2 & 0.0 & 0.0 \\
                           & $\bar{Q}_{\mathrm{best}}-\bar{Q}_{\mathrm{mean}}$ & 0.105 & 0.137 & 0.166 \\
\bottomrule
\end{tabular}
\end{table}

\subsection{Leave-One-Task-Out Calibration}
\label{app:loto}

Suite-level calibration in Table~\ref{tab:eval} selects a single guidance strength using all tasks within a suite. While this evaluates whether one shared strength can work across heterogeneous tasks, each task still contributes to the selection, so it does not directly measure transfer to a task unseen during calibration.

To evaluate this setting, we conduct a leave-one-task-out (LOTO) test: for each task in a suite, the guidance strength is selected to maximize the average success rate over the remaining tasks and evaluated on the held-out task, with results averaged over all held-out tasks. We reuse the per-task success rates from the guidance-strength sweep in Appendix~\ref{app:weights}, and each method selects within its own sweep range.

Table~\ref{tab:loto} reports the results. On the LIBERO and RoboCasa atomic tasks, AGF outperforms QGF on every held-out suite. Its advantage comes from stability rather than peak performance: on most suites AGF selects the same strength regardless of which task is held out, so excluding a task from calibration costs it almost nothing, whereas QGF loses a larger share of its task-level gains once per-task selection is unavailable, which is consistent with its fluctuating response to the guidance strength (Figure~\ref{fig:weight}). On LIBERO-Pro with MolmoAct2 (right block), the two methods remain comparable under LOTO, mirroring their suite-level parity in Table~\ref{tab:eval}, and both degrade only mildly from task-level calibration. LIBERO-10 is the exception in both settings, where task heterogeneity makes any single shared strength less effective. Overall, AGF is ahead on the large majority of the nine LOTO comparisons, and the results support suite-level calibration as a practical proxy for deployment to unseen tasks without task-specific strength selection.

\begin{table}[t]
\centering
\caption{
\textbf{Leave-one-task-out (LOTO) calibration.}
Success rate (\%) when the guidance strength is selected on all but one task and evaluated on the held-out task, averaged over held-out tasks. Bold marks the better LOTO result per suite; our method is shaded.
}
\label{tab:loto}
\footnotesize
\setlength{\tabcolsep}{3pt}
\resizebox{\textwidth}{!}{%
\begin{tabular}{llcccc|c|cccc}
\toprule
& & \multicolumn{4}{c|}{LIBERO (SmolVLA)} & \makecell{RoboCasa\\($\pi_{0.5}$)} & \multicolumn{4}{c}{LIBERO-Pro (MolmoAct2)} \\
\cmidrule(lr){3-6} \cmidrule(lr){7-7} \cmidrule(lr){8-11}
Calibration & Method & Goal & Object & Spatial & LIBERO-10 & Atomic & Goal & Object & Spatial & LIBERO-10 \\
\midrule
\multirow{2}{*}{Task-level}
& QGF & 81.4 & 95.6 & 77.6 & 39.2 & 49.6 & 77.4 & 85.8 & 77.0 & 69.4 \\
& \cellcolor{oursbg}AGF & \cellcolor{oursbg}83.2 & \cellcolor{oursbg}95.4 & \cellcolor{oursbg}75.6 & \cellcolor{oursbg}42.0 & \cellcolor{oursbg}49.4 & \cellcolor{oursbg}77.8 & \cellcolor{oursbg}84.6 & \cellcolor{oursbg}79.0 & \cellcolor{oursbg}68.8 \\
\midrule
\multirow{2}{*}{Suite-level}
& QGF & 80.6 & 93.6 & 73.0 & 36.0 & 46.3 & 76.0 & 84.6 & 76.0 & 67.8 \\
& \cellcolor{oursbg}AGF & \cellcolor{oursbg}81.0 & \cellcolor{oursbg}93.4 & \cellcolor{oursbg}74.6 & \cellcolor{oursbg}36.8 & \cellcolor{oursbg}46.6 & \cellcolor{oursbg}76.4 & \cellcolor{oursbg}83.8 & \cellcolor{oursbg}77.4 & \cellcolor{oursbg}67.0 \\
\midrule
\multirow{2}{*}{LOTO}
& QGF & 80.6 & 91.8 & 69.6 & 31.8 & 46.3 & 74.8 & \textbf{84.6} & 76.0 & \textbf{67.8} \\
& \cellcolor{oursbg}AGF & \cellcolor{oursbg}\textbf{81.0} & \cellcolor{oursbg}\textbf{93.4} & \cellcolor{oursbg}\textbf{74.6} & \cellcolor{oursbg}\textbf{33.6} & \cellcolor{oursbg}\textbf{46.6} & \cellcolor{oursbg}\textbf{75.2} & \cellcolor{oursbg}83.0 & \cellcolor{oursbg}\textbf{77.4} & \cellcolor{oursbg}65.4 \\
\bottomrule
\end{tabular}%
}
\end{table}

\subsection{Paired Significance Tests}
\label{app:stats}

All comparisons use the suite-level calibration weights of Table~\ref{tab:guidance_weights} and are episode-level paired tests: both methods are evaluated on identical episode seeds on all three benchmarks{:} LIBERO (10 tasks $\times$ 50 episodes per suite, SmolVLA), RoboCasa (18 atomic tasks $\times$ 50 episodes, $\pi_{0.5}$), and LIBERO-Pro (10 tasks $\times$ 5 conditions $\times$ 10 episodes per suite, MolmoAct2){,} with a single training seed throughout; AGF denotes the $M{=}4$ variant as in the main text.
Each episode seed fixes both the initial action noise and the environment initialization, and for each task all methods are evaluated on the same RTX 4090 to avoid hardware-dependent nondeterminism, so that per-episode outcomes are directly comparable across methods.
We report the mean success-rate difference $\Delta$SR (percentage points, positive favors AGF), exact McNemar $p$-values, and 95\% paired bootstrap confidence intervals (20{,}000 resamples), per suite and pooled within each benchmark.
Pooled over suites, AGF's improvement over the pretrained policy is significant on LIBERO and LIBERO-Pro and directionally positive on RoboCasa, while AGF and QGF are statistically indistinguishable on every suite of every benchmark.

\begin{table}[h]
\centering
\caption{\textbf{Paired significance tests.} Episode-level paired comparisons on identical seeds across all three benchmarks; positive $\Delta$SR favors AGF.}
\label{tab:paired_stats}
\footnotesize
\setlength{\tabcolsep}{5pt}
\begin{tabular}{llccc}
\toprule
Comparison & Suite & $\Delta$SR & McNemar $p$ & 95\% CI \\
\midrule
\multicolumn{5}{l}{\textit{LIBERO (SmolVLA)}} \\
\midrule
\multirow{5}{*}{Base vs.\ AGF}
& Goal      & $+4.4$ & $0.021$ & $[+0.8, +8.0]$ \\
& Object    & $+3.8$ & $0.020$ & $[+0.8, +6.8]$ \\
& Spatial   & $+3.0$ & $0.191$ & $[-1.2, +7.2]$ \\
& LIBERO-10 & $+3.2$ & $0.195$ & $[-1.2, +7.8]$ \\
& \textbf{Pooled} & $\mathbf{+3.6}$ & $\mathbf{3.4\times10^{-4}}$ & $\mathbf{[+1.7, +5.6]}$ \\
\cmidrule(lr){1-5}
\multirow{5}{*}{QGF vs.\ AGF}
& Goal      & $+0.4$ & $0.913$ & $[-3.2, +4.0]$ \\
& Object    & $-0.2$ & $1.000$ & $[-2.6, +2.2]$ \\
& Spatial   & $+1.6$ & $0.526$ & $[-2.8, +6.0]$ \\
& LIBERO-10 & $+0.8$ & $0.801$ & $[-3.8, +5.4]$ \\
& \textbf{Pooled} & $+0.7$ & $0.542$ & $[-1.3, +2.6]$ \\
\midrule
\multicolumn{5}{l}{\textit{RoboCasa ($\pi_{0.5}$)}} \\
\midrule
Base vs.\ AGF & Atomic & $\mathbf{+2.6}$ & $\textbf{0.075}$ & $[\mathbf{-0.2, +5.2}]$ \\
QGF vs.\ AGF  & Atomic & $+0.2$ & $0.939$ & $[-2.6, +3.0]$ \\
\midrule
\multicolumn{5}{l}{\textit{LIBERO-Pro (MolmoAct2)}} \\
\midrule
\multirow{5}{*}{Base vs.\ AGF}
& Goal      & $+1.0$ & $0.359$ & $[-0.6, +2.8]$ \\
& Object    & $+0.8$ & $0.424$ & $[-0.6, +2.2]$ \\
& Spatial   & $+1.6$ & $0.022$ & $[+0.4, +3.0]$ \\
& LIBERO-10 & $+1.2$ & $0.286$ & $[-0.6, +3.0]$ \\
& \textbf{Pooled} & $\mathbf{+1.1}$ & $\mathbf{5.9\times10^{-3}}$ & $\mathbf{[+0.4, +1.9]}$ \\
\cmidrule(lr){1-5}
\multirow{5}{*}{QGF vs.\ AGF}
& Goal      & $+0.4$ & $0.839$ & $[-1.6, +2.2]$ \\
& Object    & $-0.8$ & $0.388$ & $[-2.2, +0.6]$ \\
& Spatial   & $+1.4$ & $0.167$ & $[-0.2, +3.2]$ \\
& LIBERO-10 & $-0.8$ & $0.557$ & $[-2.8, +1.2]$ \\
& \textbf{Pooled} & $+0.0$ & $1.000$ & $[-0.8, +0.9]$ \\
\bottomrule
\end{tabular}
\end{table}

\subsection{Guidance Weight Selection}
\label{app:weights}

For each method we use a single guidance weight per benchmark suite, selected by the suite-level average success rate. 
The guidance term of each method is normalized by the magnitude of the flow velocity predicted by the pretrained VLA, so that all methods share the common sweep range $[0.25, 4.0]$; the selected values are listed in Table~\ref{tab:guidance_weights}.

\begin{table}[t]
\centering
\caption{
\textbf{Selected guidance hyperparameters.}
Suite-level selected guidance weight $w$ for each gradient-based method and the selected number of candidates $N$ for Q-BoN. All gradient-based methods share the sweep range $[0.25, 4.0]$ after flow-velocity normalization; Q-BoN sweeps $N \in \{4, 8, 16\}$. Ties are broken toward the smaller weight.
}
\label{tab:guidance_weights}
\footnotesize
\setlength{\tabcolsep}{4pt}
\resizebox{\textwidth}{!}{%
\begin{tabular}{lcccc|c|cccc}
\toprule
& \multicolumn{4}{c|}{LIBERO (SmolVLA)} & \makecell{RoboCasa\\($\pi_{0.5}$)} & \multicolumn{4}{c}{LIBERO-Pro (MolmoAct2)} \\
\cmidrule(lr){2-5} \cmidrule(lr){6-6} \cmidrule(lr){7-10}
Method & Goal & Object & Spatial & LIBERO-10 & Atomic & Goal & Object & Spatial & LIBERO-10 \\
\midrule
QDPS            & 4.0 & 0.25 & 4.0 & 1.0 & 4.0  & 2.0 & 1.0 & 2.0  & 1.0 \\
QGF             & 4.0 & 1.0  & 2.0 & 1.0 & 0.25 & 2.0 & 0.5 & 0.25 & 0.25 \\
Q-BoN ($N$)     & 8   & 8    & 4   & 4   & 4    & 4   & 4   & 4    & 4 \\
\midrule
\rowcolor{oursbg}
AGF ($M{=}1$)   & 4.0 & 2.0 & 2.0 & 2.0 & 2.0 & 1.0 & 2.0 & 2.0 & 0.5 \\
\rowcolor{oursbg}
AGF ($M{=}4$)   & 4.0 & 2.0 & 4.0 & 4.0 & 4.0 & 2.0 & 4.0 & 1.0 & 2.0 \\
\bottomrule
\end{tabular}%
}
\end{table}

\begin{table}[h]
\centering
\caption{
\textbf{Subtask success on \textit{open-pnp-close}} (25 paired episodes; each subtask judged independently, so a later subtask can succeed after an earlier one fails, e.g., closing the box without placing the ball).
}
\label{tab:real_subtask}
\footnotesize
\begin{tabular}{lccc}
\toprule
Subtask & Base & QGF & AGF \\
\midrule
Open the box & 15/25 & 18/25 & \textbf{25/25} \\
Place the ball inside & 12/25 & 15/25 & \textbf{21/25} \\
Close the box & 17/25 & 14/25 & \textbf{20/25} \\
\midrule
All three (task success) & 11/25 & 12/25 & \textbf{19/25} \\
\bottomrule
\end{tabular}
\end{table}

\subsection{Subtask analysis for real robot long task}
We evaluate AGF on the real-robot system with two tasks: a short pick-and-place (\textit{pnp-plate}) and a long compositional task (\textit{open-pnp-close}) consisting of three subtasks: open the box, place the ball inside, and close the box.
For the compositional task, Table~\ref{tab:real_subtask} scores each subtask independently over the 25 paired episodes.
AGF improves every subtask over the base policy and completes all 25 opening attempts, while the base policy frequently skips or fails intermediate subtasks (e.g., closing the box without placing the ball), which the independent scoring makes visible.
QGF's transferred strength improves opening and placing only marginally and does not improve closing, consistent with its overall lack of gain on this task (Table~\ref{tab:real_robot}).

\subsection{Qualitative comparison for Real robot experiments}
\label{app:robot_qual}
 
We show paired rollouts of Base VLA, QGF, and AGF on both evaluation tasks,
using identical initial conditions within each task.

\textbf{\textit{pnp-plate}.}
Figure~\ref{fig:real_qual_2} shows a different failure mode. All three methods
grasp the ball successfully, so the gap does not arise at the picking stage.
Base VLA and QGF, however, release the ball in transit: the gripper opens
before reaching the plate, and neither recovers within the episode time limit.
AGF transports the ball without dropping it and completes the task. This is
consistent with the subtask breakdown in Table~\ref{tab:real_subtask}, where
AGF's gain concentrates in the transport-and-place stage rather than in
grasping.
 
\textbf{\textit{open-pnp-close}.}
Figure~\ref{fig:real_qual_1} shows that the three methods diverge at the very
first subtask. Base VLA fails all three: it never opens the box, and the
episode ends without the ball being transported or the box closed. QGF opens
nothing either, yet proceeds to pick and place the ball, so the sequence is
executed out of order and the episode is scored as a failure. AGF completes
the three subtasks in order (opening the box, placing the ball inside, and
closing it) within the time limit.

\begin{figure}[t]
    \centering
    \includegraphics[width=\linewidth]{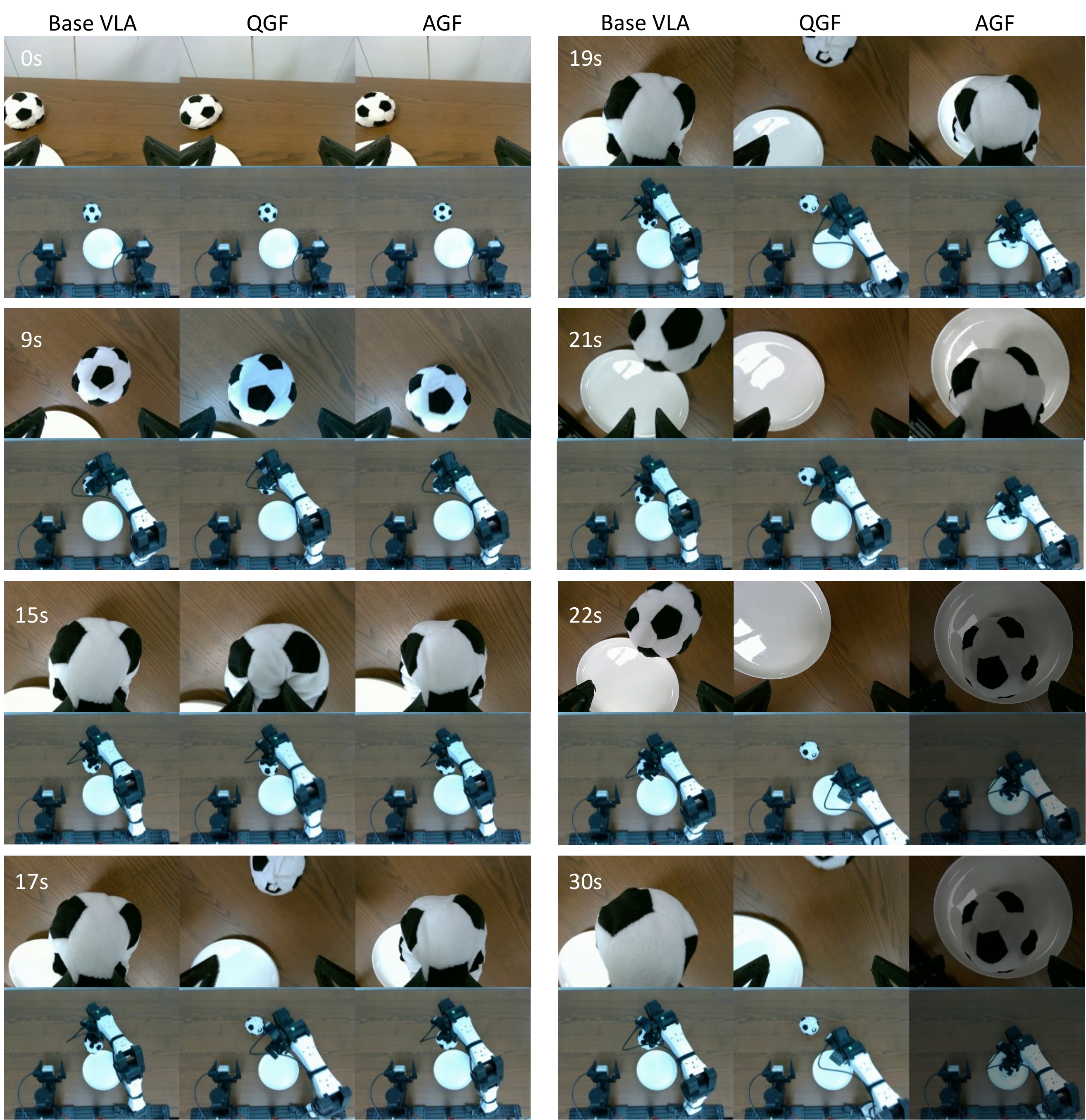}
    \caption{Qualitative rollouts on \textit{pnp-plate}. Eight timestamps from paired episodes of Base VLA, QGF, and AGF, all initialized identically. Within each block the top row is the wrist camera and the bottom row the top-view camera. Frames are dimmed after a method has completed the task.}
    \label{fig:real_qual_2}
\end{figure}

\begin{figure}[t]
    \centering
    \includegraphics[width=\linewidth]{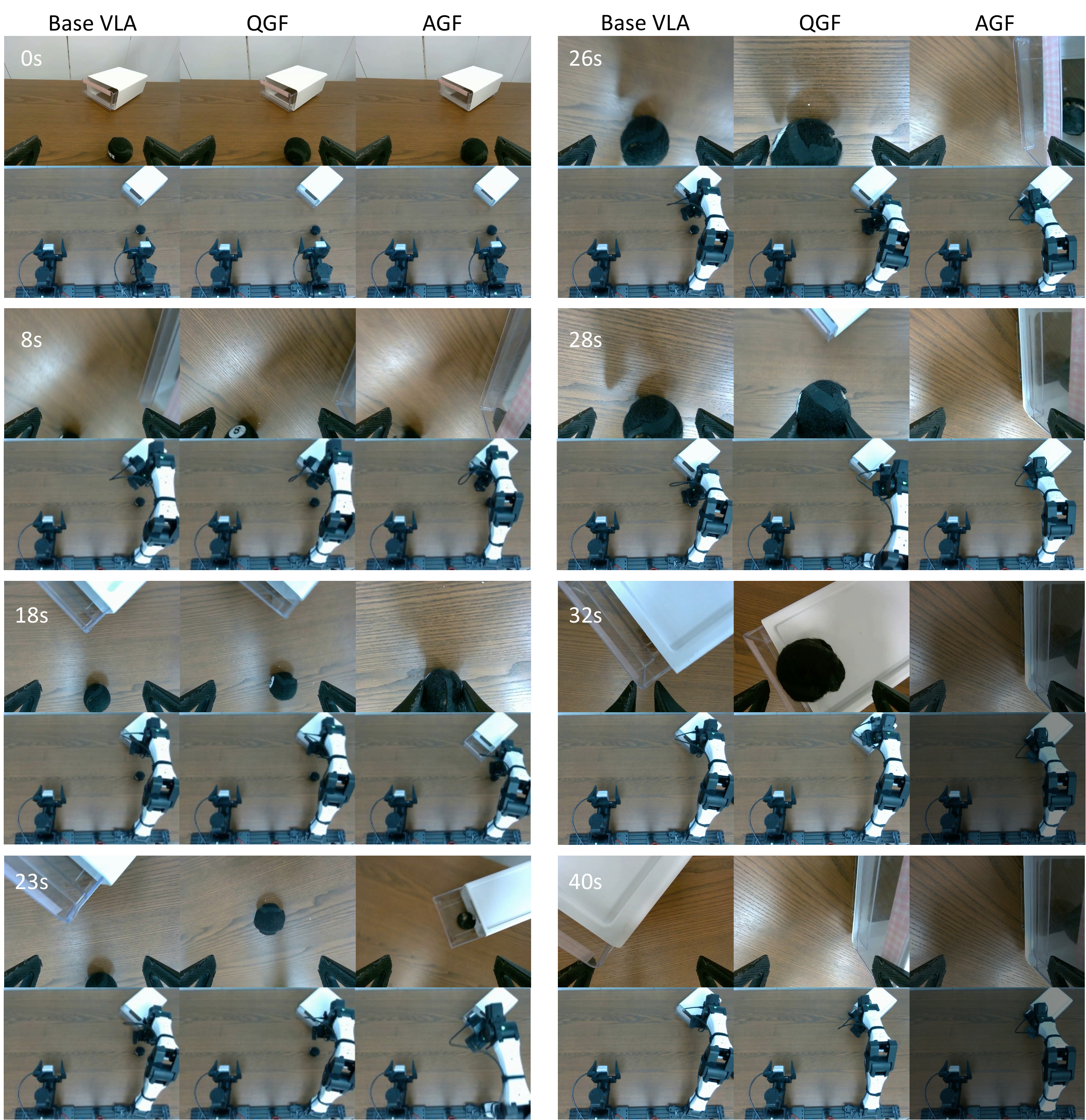}
    \caption{Qualitative rollouts on \textit{open-pnp-close}. Eight timestamps from paired episodes of Base VLA, QGF, and AGF, all initialized identically. Within each block the top row is the wrist camera and the bottom row the top-view camera. Frames are dimmed after a method has completed the task.}
    \label{fig:real_qual_1}
\end{figure}

\section{Limitations}
\label{app:limitations}
AGF inherits two dependencies from its design.
First, following the single-task critic setting of prior work, each task requires training its own critic ensemble and guidance network{. The critic is shared with every guidance baseline; the guidance network is the additional cost AGF pays to remove the critic from deployment, paid once per task rather than once per episode. Amortizing} guidance across tasks with a shared network remains an open direction.
Second, the quality of the guidance is bounded by the quality of the critic: AGF distills the critic's gradient signal and cannot correct for a poorly calibrated value function. 
Finally, on suites with highly heterogeneous tasks such as LIBERO-10, any single deployed guidance strength degrades from per-task calibration (Appendix~\ref{app:loto}), suggesting that strength selection itself could benefit from state-dependent adaptation.

\end{document}

%% file: math_commands.tex
\usepackage{amsmath,amsfonts,bm}

\def\eqref#1{equation~\ref{#1}}

\def\1{\bm{1}}

\def\rmI{{\mathbf{I}}}

\def\va{{\bm{a}}}

\def\vd{{\bm{d}}}

\def\vf{{\bm{f}}}
\def\vg{{\bm{g}}}

\def\vs{{\bm{s}}}

\def\vu{{\bm{u}}}
\def\vv{{\bm{v}}}

\def\vz{{\bm{z}}}

\def\veps{{\bm{\epsilon}}}
\def\vlamb{{\bm{\lambda}}}

\def\mA{{\bm{A}}}
\def\mB{{\bm{B}}}

\def\mE{{\bm{E}}}

\DeclareMathAlphabet{\mathsfit}{\encodingdefault}{\sfdefault}{m}{sl}
\SetMathAlphabet{\mathsfit}{bold}{\encodingdefault}{\sfdefault}{bx}{n}

\DeclareMathOperator*{\argmax}{arg\,max}